\documentclass[a4paper]{article}[12pts]

\usepackage[english]{babel}
\usepackage[utf8x]{inputenc}
\usepackage[T1]{fontenc}
\usepackage{enumitem}

\normalsize

\usepackage[a4paper,top=1in,bottom=1in,left=1in,right=1in,marginparwidth=1in]{geometry}

\usepackage{setspace}
\usepackage{amsmath}
\usepackage{amssymb} 
\usepackage{amsthm}
\usepackage{array}
\usepackage{amsfonts, mathtools}
\usepackage{algorithm,algpseudocode}
\usepackage{bm}
\usepackage{bbm}
\usepackage{comment}
\usepackage{graphicx}
\usepackage{multirow, booktabs}
\usepackage{makecell}
\usepackage{caption}
\usepackage{subcaption}
\usepackage[colorinlistoftodos]{todonotes}
\usepackage[colorlinks=true, allcolors=blue]{hyperref}
\usepackage{thmtools}
\usepackage{thm-restate}
\usepackage{hyperref}
\usepackage{cleveref}

\newcolumntype{P}[1]{>{\centering\arraybackslash}p{#1}}

\usepackage{siunitx} 

\newcommand{\bmu}{\boldsymbol{\mu}}
\newcommand{\E}{\mathbb{E}}

\newtheorem{lemma}{Lemma}
\newtheorem{example}{Example}
\newtheorem{assumption}{Assumption}
\newtheorem{theorem}{Theorem}

\newtheorem{corollary}{Corollary}
\newtheorem{definition}{Definition}
\newtheorem{remark}{Remark}

\newcommand{\bz}{\mathbf{z}}
\newcommand{\boe}{\boldsymbol{\eta}}

   \newtheoremstyle{TheoremNum}
        {\topsep}{\topsep}              
        {\itshape}                      
        {}                              
        {\bfseries}                     
        {.}                             
        { }                             
        {\thmname{#1}\thmnote{ \bfseries #3}}
    \theoremstyle{TheoremNum}

   \newtheoremstyle{TheoremNum}
        {\topsep}{\topsep}              
        {\itshape}                      
        {}                              
        {\bfseries}                     
        {.}                             
        { }                             
        {\thmname{#1}\thmnote{ \bfseries #3}}
    \theoremstyle{TheoremNum}

\usepackage{natbib}

\definecolor{rose}{rgb}{1.0, 0.33, 0.64}

\title{Users as Annotators: LLM Preference Learning from Comparison Mode}
\author{Zhongze Cai, \ \  Xiaocheng Li}
\date{\small
Imperial College Business School, Imperial College London
}

\begin{document}
\maketitle
\onehalfspacing

\begin{abstract}
Pairwise preference data have played an important role in the alignment of large language models (LLMs). Each sample of such data consists of a prompt, two different responses to the prompt, and a binary label indicating which of the two responses is better. The labels are usually annotated by professional human annotators. In this paper, we consider an alternative approach to collect pairwise preference data -- user annotation from comparison mode. With the increasingly wider adoption of LLMs among the population, users are contributing more and more of their preference labels through their daily interactions with the LLMs. The upside of such labels is that users are the best experts in judging the responses to their own queries/prompts, but the downside is the lack of quality control in these labels. In this paper, we consider a new idea of generating two responses from two different models or two different versions of the same model. The asymmetry allows us to make an inference of the user's data quality through our proposed user behavior model. We develop an expectation-maximization algorithm to estimate a latent quality factor of the user, and filter users' annotation data accordingly. The downstream task shows the effectiveness of our approach in both capturing the user behavior and data filtering for LLM alignment.
\end{abstract}

\def\thefootnote{}\relax\footnotetext{Corresponding to Zhongze Cai (\url{z.cai22@imperial.ac.uk}).}

\section{Introduction}

The development of large language models (LLMs) relies on human-labeled data to align model outputs with user preferences and human values. This alignment can take multiple forms. In supervised fine-tuning (SFT), annotators provide high-quality demonstrations of responses to user prompts. This enables the model to directly imitate expert behavior \citep{ouyang2022training, touvron2023llama}. Another more scalable approach is to collect preference data, where annotators compare pairs of model outputs and indicate which one they prefer. Such preference-based alignment methods include reinforcement learning from human feedback (RLHF) \citep{christiano2017deep, bai2022training} and direct preference optimization (DPO) \citep{rafailov2023direct}. Major LLM companies employ teams of full-time annotators to label preference data for alignment pipelines like RLHF and DPO. While effective, this approach is costly and inherently limited in scale due to the expense of expert annotators. 

In this paper, we consider an alternative approach -- to collect preference feedback directly from end users of the LLMs. For example, OpenAI’s ChatGPT features a \textit{comparison mode} (see Figure \ref{fig:demo_prefer_this}), in which users are occasionally shown two responses to the same query and asked to select their preferred one \citep{achiam2023gpt}. This approach leverages the vast user base to collect a huge amount of data. It also provides an accessible way for engaged users to express their preferences, with minimal additional effort beyond their normal interactions with the LLM agent.

\begin{figure}[ht!]
\centering
\includegraphics[width=0.7\linewidth]{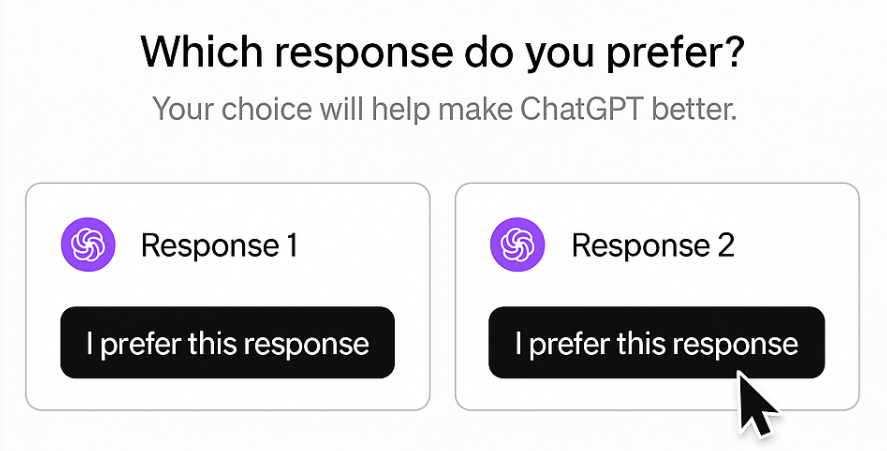}
\caption{ChatGPT's comparison mode. Two responses are generated, and the user clicks the preferred one.}
\label{fig:demo_prefer_this}
\end{figure}

However, user-collected feedback faces a critical limitation: a lack of quality control. Users are neither incentivized nor obligated to provide consistent judgments, which makes it difficult to separate high-quality feedback from noisy or careless inputs. Expert validation is impractical at scale, as it would require reviewing each unique prompt–response pair. Likewise, common quality-control techniques such as golden questions or self-consistency checks \citep{liu2025humans} are unsuitable in this setting, since platforms cannot control user queries or require repeated annotations. 


In this paper, we introduce a framework to assess the quality of user-annotated preferences and filter the alignment data accordingly. The pipeline is illustrated in Figure \ref{fig:workflow}.

\begin{itemize}
\item To our knowledge, we propose the first probabilistic model to capture users' behavior in annotating preference data in the comparison mode. 
\item We introduce a slight twist of the existing preference pipeline in that we have the two responses generated from two different LLMs. Based on this setup, we develop an EM algorithm that estimates the parameters of the user behavioral model and infers the user's attentiveness level as a posterior. 
\item Based on the inferred attentiveness level, we construct alignment datasets and test the idea of different LLMs. Our numerical experiments illustrate the effectiveness of both the EM algorithm and the pipeline in terms of improving the alignment performance. 
\end{itemize}

We will discuss some related works throughout the paper and defer more discussions to Appendix \ref{sec:related_lit}.

\begin{figure}[ht!]
\centering
\includegraphics[width=0.6\linewidth]{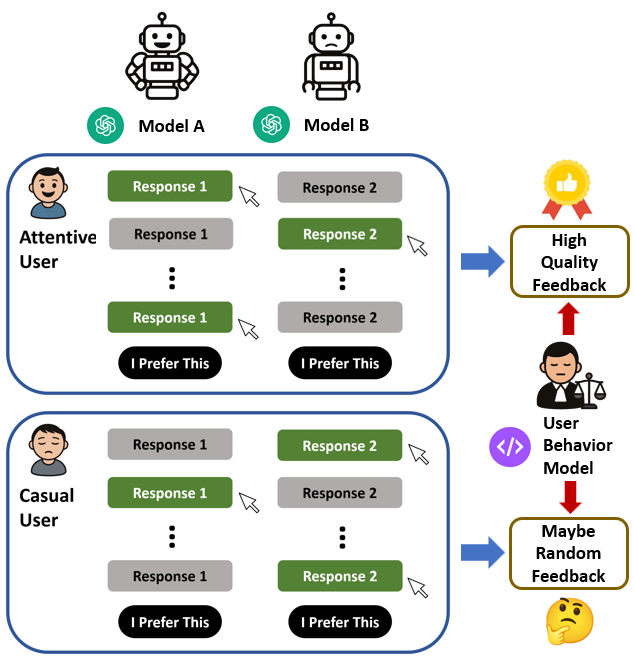}
\caption{The proposed pipeline estimates the quality of users’ preference labels in the comparison mode. For each prompt, two responses are generated from different LLMs. Each user’s preference history is tracked over time. When Model A is more powerful than Model B, \textit{attentive users} are expected to favor Model A more often, whereas \textit{casual users} might choose between the two models with equal probability. These preference patterns are then used to fit a user behavior model and infer the attentiveness level of each user. Then only the labels from the attentive users will be used for the downstream alignment tasks.}
\label{fig:workflow}
\end{figure}

\section{Reward Modeling and User Behavioral Model}

\subsection{Reward Modeling}\label{sec:RM_Modeling}

Reward modeling (or reward models) is one key component in the alignment/post-training of LLMs \citep{bai2022training,ouyang2022training}. The training of a reward model (RM) generally relies on a dataset of 
$$\mathcal{D} = \{(x_{i}, y_{i1}, y_{i2}, z_i)\}_{i=1}^n.$$
In some cases, there might be exceptions, such as rule-based or LLM-judge RMs, but our paper will mainly focus on the so-called neural RM trained on a dataset like $\mathcal{D}$. Here $x_i\in \mathcal{X}$ denotes a prompt, $y_{i1},y_{i2}\in\mathcal{Y}$ denote two candidate responses to $x_i$, and the label $z_i\in \{0,1\}$ denotes the human preference between the two responses ($z_i=1$ if $y_{i1}$ is better and $z_i=0$ if $y_{i2}$ is better). The RM $r:\mathcal{X}\times \mathcal{Y} \rightarrow \mathbb{R}$ is trained from $\mathcal{D}$ in a way that it assigns a score $r(x,y)$ for a prompt-response pair $(x,y)$. Ideally, a high score of $r(x,y)$ indicates $y$ is a good response to $x$, and then one can use $r(x,y)$ to align the LLM -- increasing the probability of generating a response $y$ with a high $r(x,y)$.

Thus, the signal origin of all the LLM alignments is the dataset $\mathcal{D}.$ In the canonical setup, the two responses can be generated as two random trials from the current to-be-aligned LLM 
\begin{equation}
y_{i1}, y_{i2}\overset{\text{i.i.d.}}\sim P(\cdot|x_i)
\label{eqn:y_12_same_model}
\end{equation}
where $P$ denotes the probability distribution specified by the LLM, and the two responses are two i.i.d. generations. Then $z_i$ is usually labeled by some \textit{human annotators} to indicate a preference between the two responses. 

\textbf{Users as annotators.} With the skyrocketing
usage of LLMs, one issue is that the prompt $x_i$ becomes more diverse, possibly covering topics that human annotators might not be familiar with. Meanwhile, users have much better domain knowledge to judge the qualities of the two responses, and they are in fact doing so in the comparison mode as Figure \ref{fig:demo_prefer_this}. However, one problem is that users are not incentivized to produce the correct label $z_i$; they may casually randomly label $z_i$ from $\{0,1\}$. Comparatively, for human annotators hired by the LLM company or the data company, one may use various quality-control or monitoring strategies to ensure that the annotators are fully committed when they label samples. In this paper, we try to address the problem with a new annotation system to enjoy the best of both worlds: good domain knowledge and high-quality annotations.

The key design of our framework (also as displayed in Figure \ref{fig:workflow}) is to have the two responses generated from two different models:
\begin{equation}
y_{i1}\sim P_{A}(\cdot|x_i), \ \ y_{i2}\sim P_{B}(\cdot|x_i)
\label{eqn:y_12_diff_model}
\end{equation}
where $P_{A}$ and $P_{B}$ denote two different LLMs. Compared to (\ref{eqn:y_12_same_model}), the asymmetry in the response generation (\ref{eqn:y_12_diff_model}) gives us a handle to filter out and remove low-quality annotations, and thus creates a path for harnessing user-generated annotations. The two models $P_{A}$ and $P_{B}$ may come from different versions of the same model family, or two different model post-training checkpoints.

In general, the RM literature considers a probabilistic preference model. In the context of this two-model system, let
$$p(x_i) \coloneqq \mathbb{P}(y_{i1}>y_{i2}|x_i) $$
where the two responses are generated following (\ref{eqn:y_12_diff_model}). The probability should be interpreted as the proportion of the whole human population that prefers $y_{i1}$ to $y_{i2}$. 
Accordingly, the label $z_i$ should be assigned by 
\begin{equation}
\mathbb{P}(z_{i}=1|x_i) = 1- \mathbb{P}(z_{i}=0|x_i) = p(x_i)
\label{eqn:z}
\end{equation}
We note $p(x_i)$ should be interpreted as the overall probability that the model $P_{A}$ generates a better response than the model $P_{B}$ given the prompt $x_i$. When the two models $P_{A}$ and $P_{B}$ are the same, as the canonical setup of RM, $p(x_i) = 1/2.$ But when they are different, $p(x_i) \neq 1/2$ most likely. This provides a handle for data filtering. In the next section, we introduce a statistical model to capture the user's annotation behavior.

\subsection{User Behavioral Model}

The goal of formulating a behavioral model is to distinguish the attentive users who carefully select the label from the casual users who randomly pick one. 
Consider a set of users indexed by $j=1,...,m.$ Each user $j$ annotates $n_j$ samples in the comparison mode: 
$$\mathcal{D}_j = \left\{\left(x_{i}^{(j)}, y_{i1}^{(j)}, y_{i2}^{(j)}, z_{i}^{(j)}\right)\right\}_{i=1}^{n_j}.$$
As mentioned earlier, one cannot assume the users are as committed as the full-time annotators and that they follow (\ref{eqn:z}) in assigning the labels $z_{i}^{(j)}$. To capture their behavior, we consider the following annotation model
\begin{equation}
\mathbb{P}\left(z_i^{(j)}=1\big|x_i^{(j)}\right)=1/2+\eta_j\cdot \left(p(x_i^{(j)})-1/2\right)
\label{eqn:p_user}
\end{equation}
where $\eta_j\in[0,1]$ is a user-specific parameter. The two extremes: when $\eta_j=0$, user $j$ assigns the label $z_i^{(j)}$ randomly by fair coin tossing; when $\eta_j=1$, user $j$ acts as a perfect annotator that follows (\ref{eqn:z}). In this light, $\eta_j$ can be interpreted as the \textit{attentiveness} level of the user; a larger value indicates that the user is more attentive and committed to contributing high-quality annotations. Unfortunately, it is unobservable to the LLM company and has to be inferred from the data. We assume $\eta_j$ follows an unknown distribution $\mathcal{P}_{\eta}$,
$$\eta_j \sim \mathcal{P}_{\eta}.$$
Then the task of obtaining high-quality data reduces to distinguishing attentive users from casual users, or to decide for some pre-specified level $\eta^*\in(0,1)$ whether
$$\eta_{j}\ge \eta^* \text{ \ or \ } \eta_{j}< \eta^*.$$
Now the asymmetry of the two models $P_A$ and $P_B$ plays a role. We can use the annotation data to infer $\eta_j$. Surprisingly, as we will see in the next section, such an inference doesn't involve the prompt and the pair of responses, although involving them through some more complicated approaches (built upon our behavioral model) may improve the inference performance.

\section{Attentiveness Inference via Expectation Maximization}

In this section, we present our approach of inferring $\eta_j$ and how we use it to filter out casual users. We first state a property of the user annotations. The result comes from
\begin{lemma}\label{lem:simple_prob}
Suppose user $j$'s prompt distribution follows $P_{\mathrm{prompt}}^{(j)}$, the following holds under the user behavioral model (\ref{eqn:p_user})
\begin{equation}\label{eq:prob_of_zij}
\mathbb{P}\left(z_i^{(j)}=1\right)=1/2+\eta_j\cdot \left(\mu_j-1/2\right),
\end{equation}
where $\mu_j \coloneqq \mathbb{E}[p(x_i^{(j)})]$ where the expectation is taken w.r.t. $x_i^{(j)}\sim P_{\mathrm{prompt}}^{(j)}$.
\end{lemma}
Intuitively, Lemma \ref{lem:simple_prob} says that the extent to which the preference of a user $j$ for model $P_A$ deviates from $\mu_j$ is determined by the attentiveness level $\eta_j$. The technical implication of Lemma \ref{lem:simple_prob} is that we can express the likelihood of observing a user's annotations in a very tractable format. Specifically, the log-likelihood of seeing user annotations $\{z_{i}^{(j)}\}_{i=1}^{n_j}$ for $j=1,...,m$ is
\begin{equation}\label{eq:likelihood_formula_L}
     L(\theta):= \sum_{j=1}^m \log \Bigg( \int\limits_{\eta_j\in[0,1]} \exp\Big(\sum_{i=1}^{n_j} l\left(z_{i}^{(j)}\big\vert\mu_j, \eta_j\right)\Big) \text{d} \mathcal{P}_\eta(\eta_j)\Bigg)
\end{equation}
where the parameter $\theta \coloneqq (P_\eta, \bmu)$ encapsulates all the unknown parameters that specify the attentiveness distribution $\mathcal{P}_\eta$ and $\bmu\coloneqq (\mu_1,\ldots, \mu_m)$. The attentiveness level $\eta_j$ in the above expression should be interpreted as \textit{latent} variables that are unobserved, so we need to perform a marginalization with respect to them. Here the likelihood of each observation is implied from Lemma \ref{lem:simple_prob}
\[
l\left(z_{i}^{(j)}\big\vert\mu_j, \eta_j\right) := z_{i}^{(j)}\cdot \log (1/2 + \eta_j\cdot (\mu_j - 1/2)) + (1-z_{i}^{(j)})\cdot \log (1/2 - \eta_j\cdot (\mu_j - 1/2)).
\]
Now the task becomes to estimate the unknown parameters $\theta$ from the annotation data and then infer $\eta_j$ for each user by calculating the posterior.

\subsection{Parameter Estimation with EM}\label{sec:para_estimate_EM}

The log-likelihood function \eqref{eq:likelihood_formula_L} involves an integration with respect to the latent variables $\eta_j$'s, and thus can be hard to optimize directly. Yet it takes a natural form for us to apply the expectation-maximization (EM) algorithm, which alternates between taking expectation with respect to the latent variables and optimizing over the parameters.


\begin{algorithm}[ht!]
    \caption{EM for User Behavior Model}
    \label{alg:EM_UBM}
    \begin{algorithmic}[1] 
    \Statex \textbf{Input}: User annotations $\{z_{i}^{(j)}\}_{1\leq i\leq n_j, 1\leq j\leq m}$, initial parameter estimations $\theta^{(1)} = (P_\eta^{(1)}, \bmu^{(1)})$\nonumber
    \For{$t=1, 2, \ldots$, while not converge,}
    \State\label{step:compute_posterior}   For each user $j$, compute the posterior of $\eta_j$ conditional on the estimated parameters $\theta^{(t)}$:
    $$
    \eta_j \mid \theta^{(t)} \propto \exp\left(\sum_{i=1}^{n_j}l\left(z_{i}^{(j)}\mid \mu_j^{(t)}, \eta_j\right)\right)P_\eta^{(t)}(\eta_j)
    $$
    \Statex\hspace{\algorithmicindent}  \textcolor{blue}{\%\% \textit{E-step}}
    \State\label{step:E-step} Compute the complete-data log-likelihood:
    $$
        Q(\theta\mid \theta^{(t)}) := \sum_{j=1}^m \mathbb{E}_{\eta_j}\Big[ \log P_\eta(\eta_j) \mid \theta^{(t)}\Big] + \sum_{j=1}^m \sum_{i=1}^{n_j} \mathbb{E}_{\eta_j}\Big[ l(z_{i}^{(j)}\mid \mu_j, \eta_j)\mid \theta^{(t)} \Big]
    $$
    where the expectation is taken with respect to the posterior of $\eta_j$ as above
    \Statex\hspace{\algorithmicindent}  \textcolor{blue}{\%\% \textit{M-step}}
    \State\label{step:M-step} Update parameters with regularization $\mathcal{R}$:
    $$
    \theta^{(t+1)} \gets \arg\max_{\theta\in \Theta} \Big(Q(\theta\mid \theta^{(t)}) + \mathcal{R}(\theta)\Big)
    $$
    \EndFor
\State \textbf{Output}: Final parameter estimate $\theta^{(T)}$
\end{algorithmic}
\end{algorithm}

Algorithm \ref{alg:EM_UBM} presents a standard EM implementation for estimating the parameters of the user behavioral model. In the algorithm, Step \ref{step:compute_posterior} computes the posterior distribution of the latent variables $\eta_j$ given the data and the current parameter estimate $\theta_t$. Step \ref{step:E-step} computes the expected complete-data log-likelihood $Q(\theta\mid \theta^{(t)})$ by marginalizing out $\eta_j$'s according to posterior. In certain cases (like Example 1 below), we may derive a closed-form expression for $\eta_j\mid \theta^{(t)}$ and thus for the calculation of the expectation. When the closed form is not available, one can utilize classic numerical methods, such as Markov Chain Monte Carlo (MCMC), to compute the expectation. 
The Step \ref{step:M-step} of Algorithm \ref{alg:EM_UBM} is the maximization step. It updates the parameter estimate by maximizing the expected complete-data log-likelihood $Q(\theta\mid \theta^{(t)})$ over the parameter space $\Theta$. Once converged, the algorithm returns the parameter at the last iteration as the output.

We note that a regularization term $\mathcal{R}(\theta)$ is included in the maximization step, which is particularly important for addressing the \textit{unidentifiability} issue. As an example, suppose the true parameter $\mu_j=0.6$, and user $j$ has full attentiveness $\eta_j=1$; while another parameterization sets $\mu_j=1$ and $\eta_j=0.2$. Then, according to \eqref{eq:prob_of_zij}, both parameterizations would imply the same distribution for each annotation $z_i^{(j)}$. In other words, the two parameterizations are statistically indistinguishable. The role of $\mathcal{R}(\theta)$ is to incorporate prior knowledge and 
guide the optimization procedure. For instance, if it is known a priori that the overall preference probability $\mu_j$'s are typically no greater than $0.7$ (i.e., model B is not significantly worse than model A), then we may set $\mathcal{R}(\theta) = \sum_{j=1}^m \max\{\log \mathbf{1}_{[0, 0.7]}(\mu_j), -M\}$ for some large $M$, where $\mathbf{1}_{A}(\cdot)$ is the indicator function on subset $A$. Including $\mathcal{R}(\theta)$ safely excludes the second parameterization in the above, where $\mu_j=1$. More generally, a natural choice of $\mathcal{R}$ is the log-prior distribution on $\theta$, in which case maximizing $Q(\theta \mid \theta^{(t)}) + \mathcal{R}(\theta)$ coincides exactly with a maximum a posteriori (MAP) EM update (see Section 1.3 of \citet{SIG-034}).


Now we provide two examples of $\mathcal{P}_{\eta}$, and both will be revisited in the later numerical experiments. Other choices for $\mathcal{P}_\eta$ are also possible such as mixtures of Betas or logistic-normal mixture distributions. These possibilities are discussed further in Appendix \ref{apx:extended_discussion_of_EM}.


\begin{example}\label{examp:1}
    \textbf{$\mathcal{P}_\eta$ as a two-point distribution}. Consider $\mathcal{P}_\eta$ as a two-point discrete distribution:
\[
\mathcal{P}_\eta = q_1 \delta_{\underline{\eta}} + q_2 \delta_{\bar{\eta}},
\]
here $\delta_{\cdot}$ denotes a Dirac delta point mass, $q_1 + q_2=1$, and $0\leq \underline{\eta} < \bar{\eta}\leq 1$. This distribution says that there are two types of users, the casual users with attentiveness level $\underline{\eta}$ and the attentive users with attentiveness level $\bar{\eta}$. The model parameters are $\theta = (q_1, \underline{\eta}, \bar{\eta}, \bmu)$, with $q_2$ always equal to $1-q_1$. 
\end{example}


\begin{example}\label{examp:2}
\textbf{$\mathcal{P}_\eta$ as a Beta distribution}. Another example is $\mathcal{P}_\eta$ denotes a Beta distribution:
\[
\mathcal{P}_\eta = \text{Beta}(\alpha, \beta) \sim \dfrac{\eta^{\alpha-1} (1-\eta)^{\beta-1}}{B(\alpha, \beta)},
\]
with $\alpha, \beta > 1$. In this setting, attentiveness is modeled as a continuous random variable: most users have moderate attentiveness, while only a small proportion are either fully attentive or fully casual. The model parameters are $\theta = (\alpha, \beta, \bmu)$. 
\end{example}


Throughout the remainder of the paper, we make the following assumption when analyzing and evaluating the algorithm.

\begin{assumption}\label{ass:mu}
We assume there is a known $\mu$ such that $\mu_1=\cdots = \mu_m =\mu$.  
\end{assumption}
The implication of this assumption is twofold. First, over the prompt distributions of different users, the average win rate of model A over model B is the same, so that $\mu_j \equiv \mu$ for $j=1,\ldots,m$. Second, the value of $\mu$ is known. We make the following remarks on this assumption. For the first point, a sufficient condition is that the prompt distributions of different users, $P_{\mathrm{prompt}}^{(j)}$, are identical. This condition is plausible when the selected prompts are drawn from semantically similar topics. In practice, when this condition does not hold, one can still apply the inference framework by either (i) clustering users or (ii) incorporating the prompts as contextual information. We defer further discussion to Appendix~\ref{subsec:uneq_mu}. For the second point, note that $\mu$ can be viewed as a universal quantity that does not depend on the attentiveness level of any individual user. Therefore, its estimation is not tied to the estimation of user-specific quantities such as $\mathcal{P}_{\eta}$ and $\eta_j$'s. In practice, one can estimate $\mu$ using a few tens of samples, where the prompts $x_i$'s are sampled from the prompt distribution and the labels $z_i$'s are annotated by engineers or language experts. In Appendix~\ref{appdix:pref_different_ds}, we report estimates of $\mu$ across (i) several widely used preference datasets, corresponding to different prompt distributions, and (ii) different pairs of LLMs. In this sense, we believe the assumption as mild, serving primarily for analytical convenience rather than imposing a practical restriction.

The technical benefits of Assumption \ref{ass:mu} are that (i) it creates more analytical structures in both the expectation and maximization steps in the algorithm, and (ii) it enables a cleaner theoretical analysis of the algorithm. In particular, Examples~\ref{examp:1} and~\ref{examp:2} admit simple analytical forms under Assumption~\ref{ass:mu}. We emphasize though that we don't exclude the possibility of having $\mu_j$'s being different and unknown. Both the user behavior model and Algorithm \ref{alg:EM_UBM} are compatible without Assumption \ref{ass:mu}.

\subsection{Global Convergence Guarantee}

Generally, the EM algorithm only guarantees a convergence to a stationary point or a local minimum, but not to the global minimum. Interestingly, the user behavioral model considered here exhibits a nice structure that admits only one local minimum. This enables the following theoretical guarantee of Algorithm \ref{alg:EM_UBM}.
\begin{theorem}[Informal]\label{thm:simple_two_mass_converge} Let $\hat{\theta}$ denote the converged parameter estimate obtained as the output of Algorithm \ref{alg:EM_UBM} and $\theta^*$ denote the true parameter. Under Assumption \ref{ass:mu} and some regularity condition on the continuity and boundedness of $\mathcal{P}_\eta$, we have for any $\epsilon>0$,
\[\lim_{m,n_j\rightarrow \infty}\mathbb{P}(\|\hat{\theta}-\theta^*\|\ge \epsilon)\rightarrow 0.
\]
\end{theorem}
A more rigorous statement of the theorem is given in Appendix \ref{subsec:prove_theorem_two_mass}, together with its proof. We note that both Example \ref{examp:1} and Example \ref{examp:2} satisfy the regularity conditions, so the theoretical guarantees in Theorem \ref{thm:simple_two_mass_converge} hold for both examples. The result has two implications. First, it gives a peace of mind when applying the algorithm. Second, it also tells that the user behavioral model exhibits a nice analytical structure. While the result is stated in an asymptotic sense, in practice, the volume of LLM usage ensures a large $m$ and $n_j$'s generally.

In proving the theorem, the regularity conditions ensure that the parameter space $\Theta$ is bounded away from its boundaries, thereby preventing degenerate cases, and impose smoothness requirements on the density function of $\mathcal{P}_\eta$. Since the specification of $\mathcal{P}_\eta$ in Example \ref{examp:1} does not naturally come with a continuous density, proving the result for Example \ref{examp:1} requires a slightly different treatment. The overall strategy, however, is standard: first, we establish a uniform law of large numbers to show that the gradient of the log-likelihood function, $\nabla_\theta L(\theta)$, converges uniformly to its population mean. Next, we show that the population mean of the log-likelihood admits a unique stationary point. These two claims imply that the stationary points of $L(\theta)$ concentrate in a neighborhood of the true parameter $\theta^*$.

Assumption~\ref{ass:mu} provides notable convenience for the theoretical understanding of the algorithm. It avoids the additional nuance of specifying prior information on $\bmu$ in $\mathcal{R}(\theta)$, as well as the need to provide corresponding theoretical characterizations. When Assumption~\ref{ass:mu} is relaxed, the global convergence guarantee in Theorem~\ref{thm:simple_two_mass_converge} can still hold in some practical settings. Specifically, when the requirement of a ``known $\mu$'' is relaxed, we show in Theorem \ref{thm:cvg_unknown_mu} of Appendix~\ref{subsec:cvg_unknown_mu} that the algorithm can still be guaranteed to converge globally. When the requirement of equal $\mu_j$'s is relaxed, the theorem does not hold in general; in Appendix~\ref{subsec:uneq_mu}, we discuss practical ways to handle this case, under which results similar to Theorem~\ref{thm:simple_two_mass_converge} can be derived.

\subsection{Inference of User Attentiveness}

Note that Algorithm \ref{alg:EM_UBM} gives the parameter estimate $\hat{\theta}$ as its output. We need one additional step to infer the user's attentiveness
\[
\begin{aligned}
    &\eta_j \mid \hat{\theta}, \{z_i^{(j)}\}_{i=1}^{n_j} \propto \exp\left(\sum_{i=1}^{n_j}l(z_{i}^{(j)}\mid \mu, \eta_j)\right)\hat{P}_\eta(\eta_j).
\end{aligned}
\]
And then, we can apply a rule that classify user $j$ as attentive if 
\[
\mathbb{P}(\eta_j \ge \eta^* \mid \hat{\theta}, \{z_i^{(j)}\}_{i=1}^{n_j}) \ge 1-\alpha
\]
for some prespecified significance level $\alpha$($=0.05$, for example) and attentiveness level $\eta^*$, which should be determined by the context and the downstream post-training task. Then the post-training data can be obtained by filtering out the casual users
$\mathcal{D}_{\text{filtered}} = \cup_{j \text{ is attentive}} \ \mathcal{D}_j.$ In the next section, we perform LLM alignment tasks using such filtered datasets and compare the performance against the case of doing no filtering.

\section{Experiments}\label{sec:experiments}

\subsection{Synthetic EM Experiments}\label{subsec:simulation}

We first examine the performance of Algorithm \ref{alg:EM_UBM} through some simulation experiments.  Figure \ref{fig:cvg_of_EM_iteration} draws several different paths of Algorithm \ref{alg:EM_UBM} with different initial points. We can see that in both panels the estimated parameters converge to the true ones. Thus it validates the convergence result established in Theorem \ref{thm:simple_two_mass_converge}. We refer to Appendix \ref{subapx:more_simulations} for more details and more numerical experiments on the convergence, and Appendix \ref{subsec:uneq_mu} for an ablation experiment when Assumption \ref{ass:mu} is relaxed. 

\begin{figure}[htbp]
    \centering
    \begin{minipage}{0.45\columnwidth}
        \centering
        \includegraphics[width=0.7\textwidth]{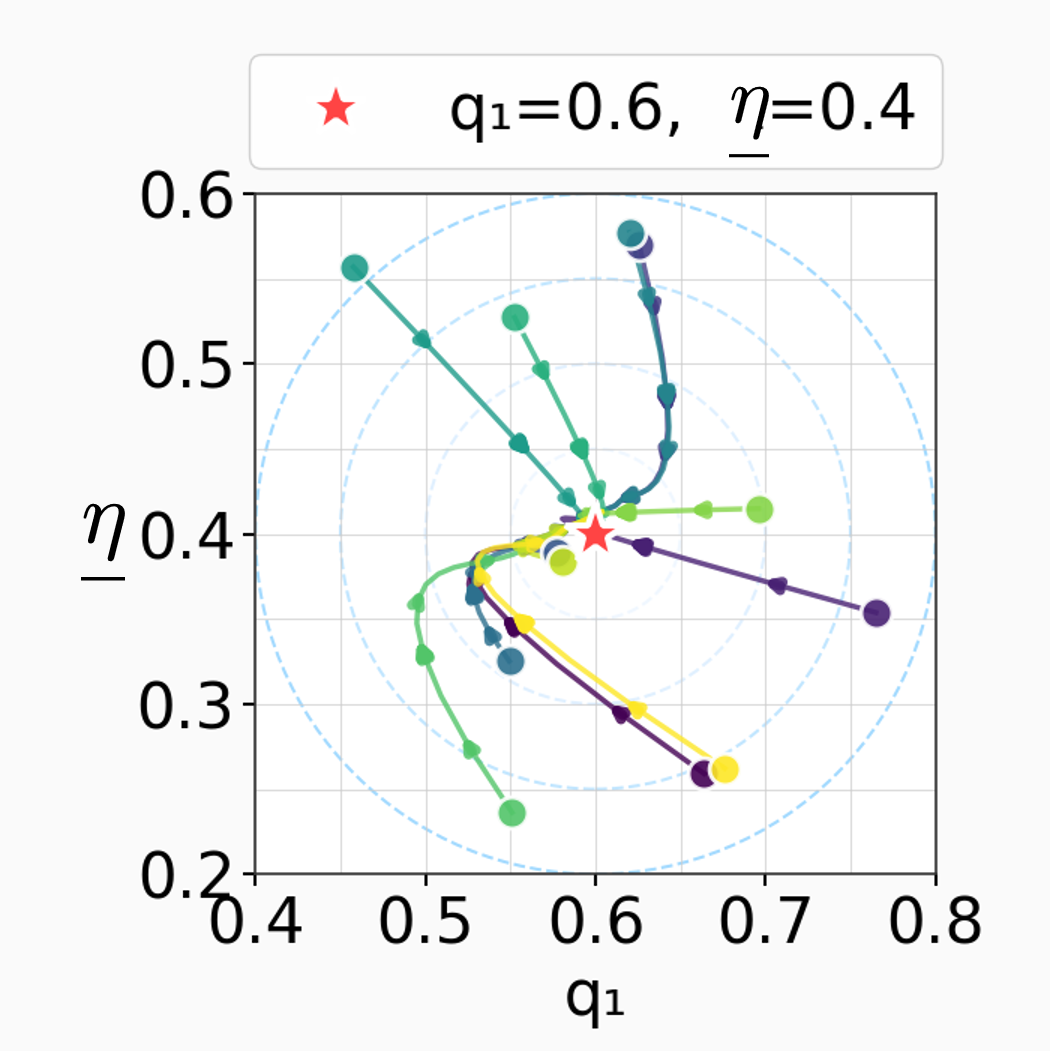}
    \end{minipage}
    \hfill
    \begin{minipage}{0.45\columnwidth}
        \centering
        \includegraphics[width=0.7\textwidth]{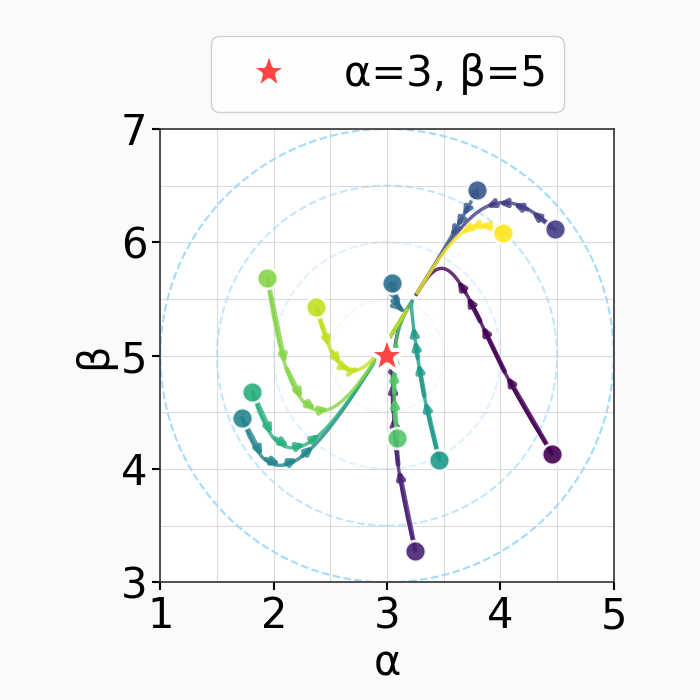}
    \end{minipage}
    \caption{The iterative updates of Algorithm \ref{alg:EM_UBM} under the model settings of Example 1 (left) and Example 2 (right). The true model for Example 1 is $\mathcal{P}_\eta = 0.6\cdot \delta_{0.4} + 0.4\cdot \delta_{0.98}$ and for Example 2 is $\mathcal{P}_\eta = \text{Beta}(3,5)$. The x- and y-axes represent the estimated parameter values, with the red star indicating the true parameters. The algorithm is repeated for multiple trials, each starting from different initialization points (marked by circles), and the trajectories of parameter updates (marked by directed arrows) are visualized. Despite different initializations, the trajectories always lead to the true parameters. 
    }
    \label{fig:cvg_of_EM_iteration}
\end{figure}

Next, we consider the case of model specification for $\mathcal{P}_{\eta}.$ When misspecification does occur, the estimates from Algorithm \ref{alg:EM_UBM} still capture a reasonable amount of information about the underlying distribution. Figure \ref{fig:model_misspecify} illustrates this result. For both panels, despite that the true distribution $\mathcal{P}_\eta$ (red lines) is not a member of the assumed model family, the estimated distribution $\hat{\mathcal{P}}_\eta$ (yellow lines) still provides a good approximation of the true distribution. In the left panel, for instance, although a discrete model cannot perfectly represent a continuous distribution, it successfully identifies the underlying clusters of attentive and casual users. This is already sufficient for our usage of data filtering. Further details on this experiment are provided in Appendix \ref{subapx:more_simulations}, where we also show that more complex model families can be used to achieve a better approximation.

\begin{figure}[ht!]
    \centering
    \begin{minipage}{0.48\columnwidth}
        \centering
        \includegraphics[width=0.7\textwidth]{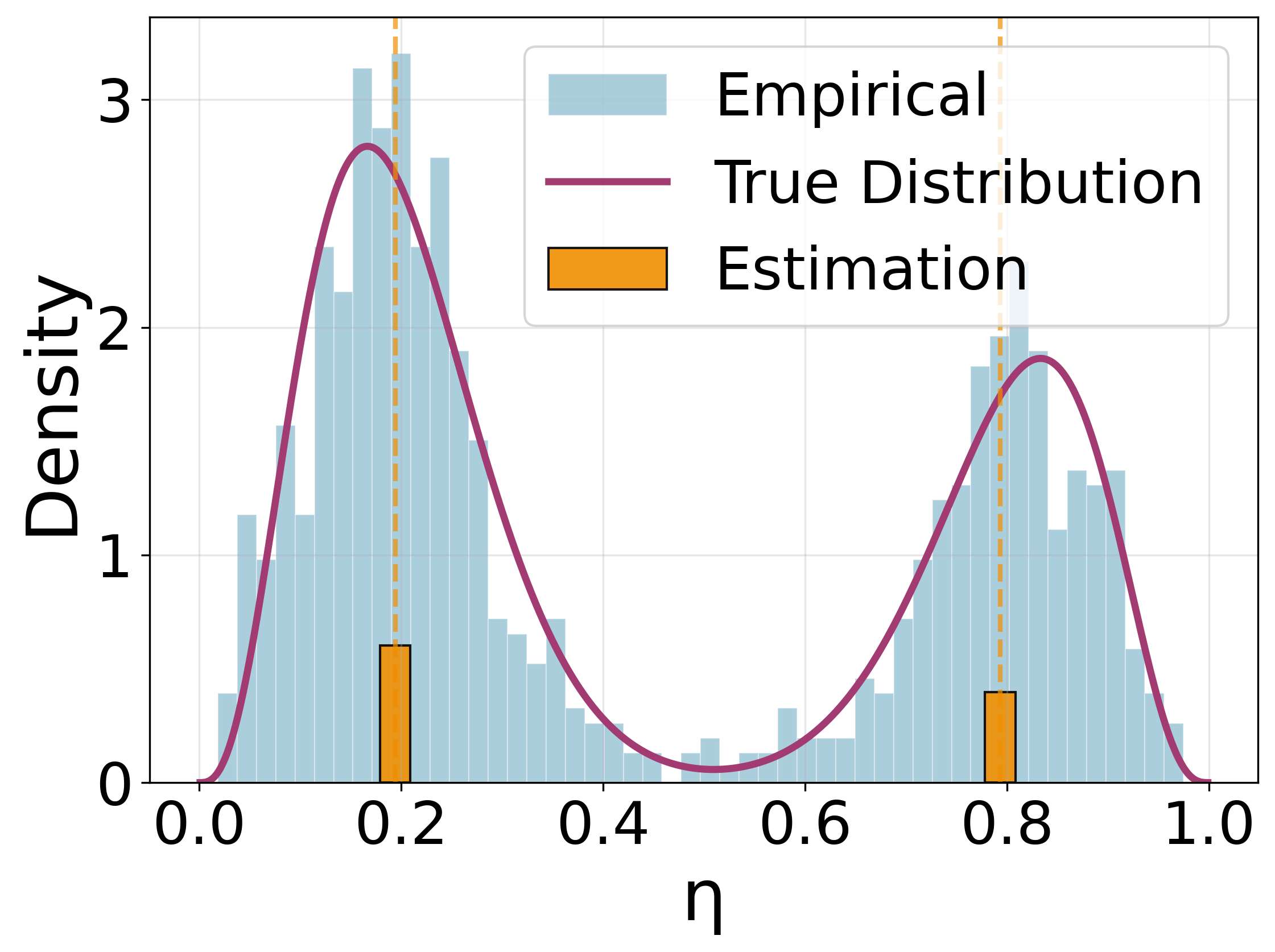}
    \end{minipage}
    \hfill
    \begin{minipage}{0.48\columnwidth}
        \centering
        \includegraphics[width=0.7\textwidth]{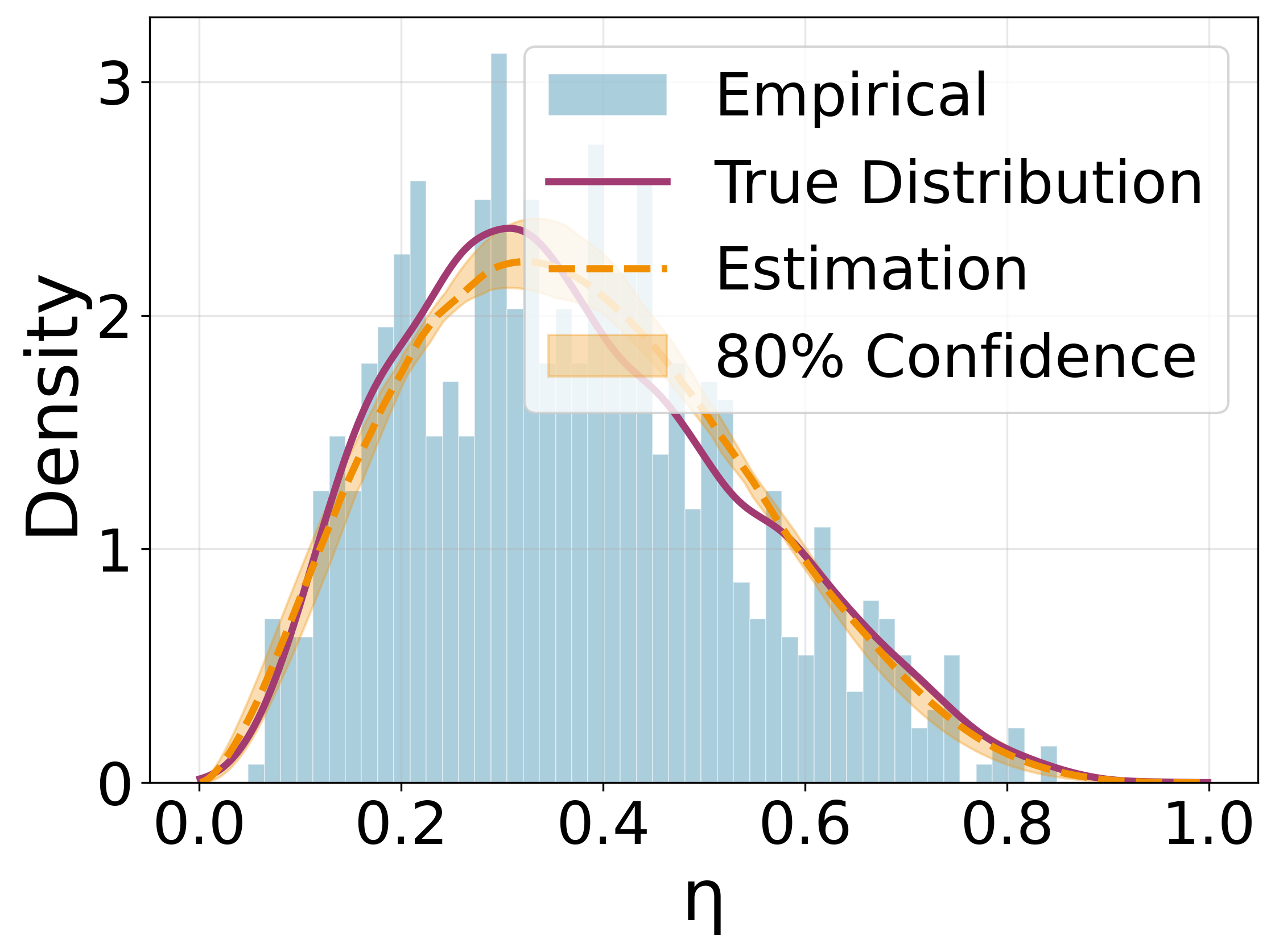}
    \end{minipage}
    \caption{The estimation of $\mathcal{P}_\eta$ under model misspecification. In the left panel, the ``True Distribution'' $\mathcal{P}_\eta$ is a mixture of two Beta distributions, while the ``Estimation'' is taken from the two-point discrete family. In the right panel, the ``True Distribution'' follows a logistic-normal distribution, while the ``Estimation'' is taken from the Beta family. In both cases, Algorithm \ref{alg:EM_UBM} approximates $\mathcal{P}_\eta$ reasonably well: in the left panel, it captures the bimodality of the true distribution, and in the right, it provides a close approximation to the true density.
    }
    \label{fig:model_misspecify}
\end{figure}

\subsection{Data Filtering and LLM Alignment}\label{subsec:LLM_DPO_Align}

Now we examine how the proposed user annotation pipeline can be used to improve the downstream performance of LLM alignment. Specifically, we perform direct preference optimization (DPO) using the dataset $\mathcal{D}_{\text{filtered}}$. The intuition behind DPO is to fine-tune an LLM on a preference dataset so that it assigns a higher probability to the preferred response and a lower probability to the alternative. The dataset has the same form $\mathcal{D} = \{(x_{i}, y_{i1}, y_{i2}, z_i)\}_{i=1}^n$ as introduced in Section \ref{sec:RM_Modeling}. DPO fine-tunes the model by minimizing the loss
$$
L_{\text{DPO}}(\theta) = - \sum_{i=1}^n 
\log \sigma \Big(
  \beta\, z'_i \Big(
    \log \frac{P_\theta(y_{i1}|x_i)}{P_{\rm ref}(y_{i1}|x_i)} - \log \frac{P_\theta(y_{i2}|x_i)}{P_{\rm ref}(y_{i2}|x_i)}
  \Big)
\Big).
$$
where $\sigma(\cdot)$ is the sigmoid function, $\beta$ is a temperature parameter, $z_i' = 2z_i - 1$, $P_\theta(\cdot)$ is the trainable LLM, and $P_{\text{ref}}(\cdot)$ is a fixed reference model (e.g., the checkpoint of $P_\theta$ before fine-tuning). If $y_{i1}$ is preferred to $y_{i2}$, then $z_i = 1$ and $z_i' = 1$, and the loss is minimized by increasing $P_\theta(y_{i1}| x_i)$ while decreasing $P_\theta(y_{i2} | x_i)$, and vice versa. More details on RLHF and DPO are provided in Appendix \ref{sec:review_RLHF_DPO}

For our experiments, the preference dataset is generated to mimic the user behavior model. The attentiveness levels $\{\eta_j\}_{j=1}^m$ of $m$ users are independently drawn from $\mathcal{P}_\eta$. Each user is presented with a randomly selected subset of prompt-response pairs $\{(x_i^{(j)}, y_{i1}^{(j)}, y_{i2}^{(j)})\}_{i=1}^{n_j}$, with $y_{i1}^{(j)}$ and $y_{i2}^{(j)}$ generated from different LLMs. Since we can't really hire users to annotate the data, we use a reward model $r(\cdot)$ to simulate the behavior of a fully attentive user. For each pair, the user makes an independent choice: with probability $\eta_j$, the user acts attentively, setting $z_i^{(j)}=1$ if $r(x_i^{(j)}, y_{i1}^{(j)}) \geq r(x_i^{(j)}, y_{i2}^{(j)})$ and $z_i^{(j)}=0$ otherwise; with probability $1-\eta_j$, the user acts casually, assigning $z_i^{(j)}$ to $0$ or $1$ at random with equal probability. 

In most of the experiments below, we set $m=400,n=50$ and use \texttt{Skywork-Reward-Gemma-2-27B} \citep{liu2024skywork} as the reward model $r(\cdot)$. The LLMs that we use to generate responses include \texttt{Qwen2.5-7B-Instruct}, \texttt{Qwen2.5-0.5B-Instruct} \citep{qwen2.5}, \texttt{Llama-3.1-Tulu-3-8B-SFT} \citep{lambert2024tulu3} and \texttt{Llama-3.2-1B} \citep{grattafiori2024llama}. The datasets where we select the prompts from include \texttt{UltraFeedback} \citep{cui2023ultrafeedback}, a large-scale preference dataset that collects human comparisons across multiple LLM-generated responses; and \texttt{HelpSteer3} \citep{wang2025helpsteer3}, a recent dataset designed to capture fine-grained user preference signals for controllable steering of LLM outputs. We use \texttt{Qwen2.5-7B} as the trainable LLM for DPO in the main text. Additional experiments on other datasets and LLMs, as well as all the implementation details, are provided in Appendix \ref{subapx:more_DPO}. Given the prompt dataset, responses generated by different LLMs induce distinct distributions of reward scores, as illustrated in Figure \ref{fig:diff_scores}. This indicates that the selected models have performance gaps among them, which serve as handles to filter out casual users.

\begin{figure}[h!]
\centering
\includegraphics[width=0.7\columnwidth]{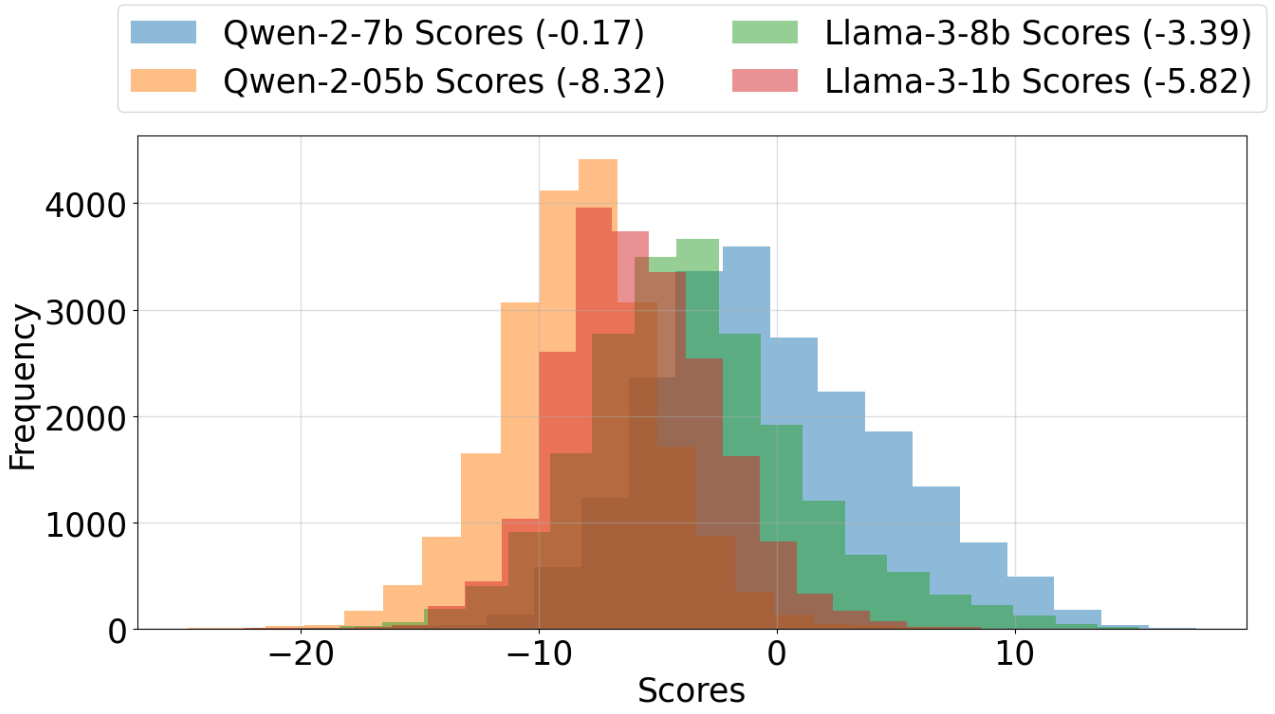}
\caption{The histogram of reward scores for different models, evaluated on the \texttt{UltraFeedback} dataset. The average score for each model is shown in brackets.}
\label{fig:diff_scores}
\end{figure}

The main results are presented in Table \ref{tab:big_table}, which demonstrates that filtering for attentive users consistently improves DPO performance. We use \texttt{Qwen2.5-7B-Instruct} and \texttt{Qwen2.5-0.5B-Instruct} as the LLMs for generating responses across different datasets. In the table, ``$\mathcal{P}_\eta$ Type = 1'' denotes $\mathcal{P}_\eta =0.2 \cdot \delta_{0.2} + 0.8\cdot \delta_{0.98}$, and ``$\mathcal{P}_\eta$ Type = 2'' denotes $\mathcal{P}_\eta =\text{Beta}(6,2)$. In the ``Model'' column, ``DPO'' represents training on the full dataset, whereas ``DPO + Filter'' represents training after filtering for users ranked top 80\% based on the inferred attentiveness. The results of the table show that filtering for attentive users consistently (using $\mathcal{D}_{\text{filtered}}$) yields improved performance compared to the unfiltered counterpart (using all the data). Despite the filtering resulting in less training data, the higher-quality subset improves average reward scores by over $10\%$ and increases the win rate over the baseline model by roughly $5\%$. 

\begin{table}[h!]
\centering
{
\resizebox{0.8\columnwidth}{!}{%
\begin{tabular}{P{2.2cm} |P{1.5cm}p{1.6cm}P{1.1cm}P{1.1cm}P{1.1cm}P{1.1cm}}
\toprule
\textbf{Dataset} & $\mathcal{P}_\eta$ \textbf{Type} & \textbf{Model} & \textbf{Score} & \textbf{Winrate} & \textbf{MMLU} & \textbf{GSM8K} \\
\midrule
\multirow{5}{*}{UltraFeedback} 
    & - & Base & -5.47 & 50\% & 70.4 & 87.0 \\
    \cmidrule(lr){2-7}
  & \multirow{2}{*}{1} & DPO & -3.59  &  63\% & 71.1  &  86.5 \\
  &           & DPO+Filter  &  -3.12 &  67\% &  72.0 &  89.0 \\
  \cmidrule(lr){2-7}
  & \multirow{2}{*}{2} & DPO & -4.42  &  58\% & 71.3  &  87.0 \\
  &          & DPO+Filter  & -3.75  &  64\% &  71.8 &  88.0 \\
\midrule
\multirow{5}{*}{HelpSteer3} 
    & - & Base & -8.39 & 50\% & 71.1 & 87.0\\
    \cmidrule(lr){2-7}
  & \multirow{2}{*}{1} & DPO & -6.56  & 67\%  &  70.8 &  86.5 \\
  &    & DPO+Filter  & -5.89  & 73\%  &  72.0 &  88.0 \\
  \cmidrule(lr){2-7}
  & \multirow{2}{*}{2} & DPO &  -6.91 & 67\%  &  71.0 & 87.0  \\
  &          & DPO+Filter  &  -6.83 &  68\% &  71.4 &  87.5 \\
\bottomrule
\end{tabular}%
}
} 
\caption{DPO performance on datasets with and without filtering. The ``Dataset'' column specifies the source dataset from which prompts are drawn. ``$\mathcal{P}_\eta$ Type'' indicates whether $\mathcal{P}_\eta$ follows a two-point distribution (Type = 1) or a Beta distribution (Type = 2). The ``Model'' column denotes the model to evaluate: ``Base'' refers to the baseline model before DPO training, ``DPO'' is the model trained on the unfiltered dataset, and ``DPO+Filter'' is the model trained after filtering users by the inferred attentiveness. Four evaluation metrics are reported: ``Score'' (average reward score on the test set), ``Winrate'' (win rate over the baseline model), and performance on two standard benchmarks, ``MMLU'' and ``GSM8K''.}
\label{tab:big_table}
\end{table}

In the previous experiment, we applied a filtering strategy that retained the top $80\%$ of users ranked by attentiveness. In practice, however, the optimal filtering $\eta^*$ threshold depends on the specific setting. Intuitively, if the threshold is set too low, low-quality data from casual users may contaminate the preference dataset. Meanwhile, if it is set too high, the remaining data may be insufficient to support effective DPO training. This naturally creates a trade-off. This trade-off is investigated and illustrated in Figure \ref{fig:DPO_difference_filtering}, where we explore how the optimal filtering threshold varies with the underlying distribution $\mathcal{P}_\eta$. The intuition is confirmed: when a larger fraction of users are casual (e.g., $\mathcal{P}_\eta = \text{Beta}(4,4)$), a stricter filtering strategy is preferable; whereas when most users are attentive (e.g., $\mathcal{P}_\eta = \text{Beta}(6,2)$), filtering out a small percentage suffices. Regardless of the shape of $\mathcal{P}_\eta$, there always exists a ``sweet spot'' corresponding to the optimal filtering strategy.

\begin{figure}[htbp]
    \centering
    \includegraphics[width=0.7\columnwidth]{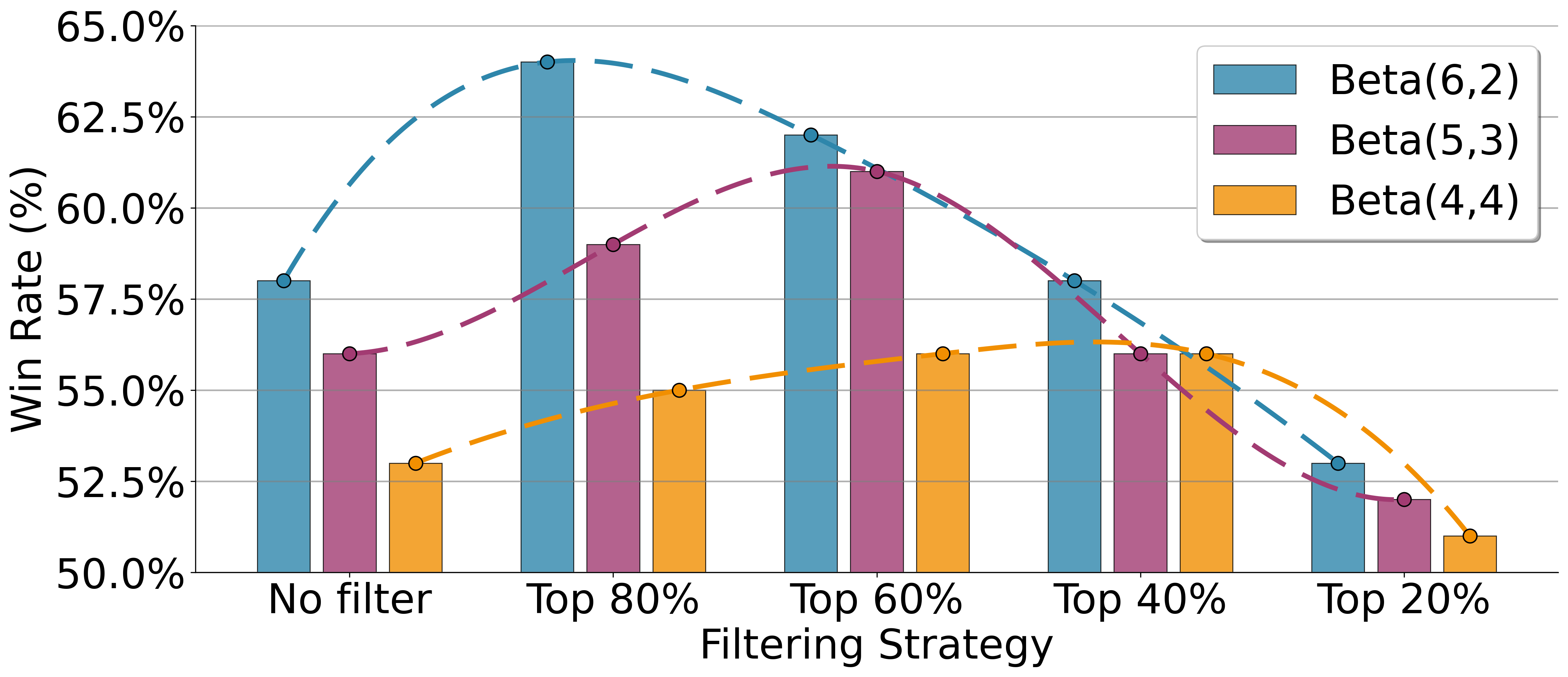}
    \caption{DPO performance under different filtering strategies. The x-axis represents filtering strategies based on the inferred attentiveness ranking, e.g., retaining the top $80\%$, $60\%$, etc. of users. The y-axis shows the win rate of the resulting DPO-trained LLM over the baseline model. The legend indicates different realizations of $\mathcal{P}_\eta$, each exhibiting its own ``sweet spot'' filtering strategy. For instance, when $\mathcal{P}_\eta = \text{Beta}(6,2)$, the optimal strategy is to retain $80\%$ of the users.
    }
    \label{fig:DPO_difference_filtering}
\end{figure}

Apart from the trade-off between data size and data quality that underlies the filtering strategy, there is another trade-off involving the average preference probability $\mu$. Our simulation study shows that larger values of $\mu$ make parameter estimation via Algorithm \ref{alg:EM_UBM} more reliable. In the context of DPO, this corresponds to a greater performance gap between the generation models A and B. However, prior work \citep{chen2024softmax} indicates that an excessively large gap between the preferred and less preferred samples can actually hinder DPO training. This creates a second trade-off, now with respect to the preference probability $\mu$. As shown in Table \ref{tab:performance_gap_and_table}, the best performance of the DPO-trained model (excluding the unattainable ``DPO+Oracle'' case) is achieved at a moderate level of $\mu = 0.74$, corresponding to the Qwen 7B \& Llama 8B setting. When $\mu$ is too large, the less preferred samples become overly weak and harm training; when $\mu$ is very close to $1/2$, attentive users and casual users become indistinguishable.

\begin{table}[ht!]
\centering
{
\resizebox{0.8\columnwidth}{!}{%
\begin{tabular}{P{2.1cm} p{1.1cm} p{1.6cm} P{1.1cm} P{1.1cm} P{1.1cm} P{1.1cm}}
\toprule
LLMs A \& B & $\mu$ & \textbf{Model} & \textbf{Score} & \textbf{Winrate} & \textbf{MMLU} & \textbf{GSM8K} \\
\midrule
- & - & Base & -5.47 & 50\% & 70.4 & 87.0 \\
\midrule
\multirow{3}{*}{\makecell{Qwen 7b \\ \& Qwen 0.5b} } &  \multirow{3}{*}{0.98} & DPO         & -4.42 & 58\% & 71.3 & 87.0 \\
& & DPO+Filter  & -3.75 & 64\% & 71.8 & 88.0 \\
& & DPO+Oracle  &  -3.70 & 67\% & 71.9 & 88.0\\
\midrule
\multirow{3}{*}{\makecell{Qwen 7b \\ \& Llama 8b}} & \multirow{3}{*}{0.74} & DPO    &  -4.29 & 60\% & 71.2 & 87.0\\
& & DPO+Filter  &  -3.71 & 64\% & 71.8 & 87.5\\
& & DPO+Oracle  & -3.65 & 67\% & 72.0 & 88.0\\
\midrule
\multirow{3}{*}{\makecell{Qwen 7b \\ \& Qwen 7b}} & \multirow{3}{*}{0.5} & DPO & -4.30 & 59\% & 71.1 & 87.5 \\
& & DPO+Filter  & -4.32 & 59\% & 71.2 & 87.5  \\
& & DPO+Oracle  &  -3.63 & 67\% & 71.8 & 88.0\\
\bottomrule
\end{tabular}%
}
\caption{DPO performance on datasets generated from different LLM pairs. The ``LLMs A \& B'' column lists the pairs of response-generating models, and the ``$\mu$'' column gives their preference probability. The ``Model'' column follows the notation of Table \ref{tab:big_table}, where ``DPO+Oracle'' refers to training with users filtered by their \textit{true} attentiveness. For smaller values of $\mu$, attentive users are harder to identify, as indicated by the widening gap between ``DPO+Filter'' and ``DPO+Oracle'', which leads to a reduced DPO performance.
}
\label{tab:performance_gap_and_table}
}
\end{table}

\section{Discussions and Conclusions}

The key component of our pipeline is the user behavioral model. In this paper, we propose a prototypical version of it which can be extended in several ways. First, it assumes that each user’s attentiveness $\eta_j$ remains constant across prompts. A natural generalization is to allow attentiveness to vary with the prompt and responses. For example, when the responses are long and detailed, users may pay less attention to carefully reading them, whereas responses with a clear right-or-wrong answer are more likely to keep users attentive. This motivates modeling attentiveness at the sample level as a function of both the user-specific features and the prompt-response ones. Such a formulation enables more fine-grained filtering strategies that improve from ``which users to filter out'' to ``which samples to filter out''.

In summary, we propose a new annotation framework in which LLM users act as annotators and provide preference feedback through the platform’s comparison mode. We develop a statistical model to capture user behavior when giving feedback, with particular emphasis on each user’s latent ``attentiveness''.  The proposed EM algorithm enjoys asymptotic global convergence, which makes the formulation of independent theoretical interest. The model proves especially useful for downstream DPO training through filtering out casual feedback and improving data quality. Moreover, it opens the door to user-incentive mechanisms, such as awarding credit points or badges for attentive feedback.

\bibliographystyle{plainnat} 
\bibliography{main.bib}

\appendix

\newpage 

\section{Related Work}
\label{sec:related_lit}
In this section, we review literature relevant to our work.

\subsection{Aligning LLMs with Human Preferences}
\label{sec:review_RLHF_DPO}
Post-training alignment of large language models (LLMs) typically involves two stages: supervised fine-tuning (SFT) with expert demonstrations, followed by preference-based optimization. The latter stage aligns models with human values by learning from pairwise comparisons of model outputs, and is dominated by two primary methods: reinforcement learning from human feedback (RLHF) \citep{christiano2017deep, bai2022training, ouyang2022training, ziegler2019fine} and, more recently, Direct Preference Optimization (DPO) \citep{rafailov2023direct}. For a comprehensive survey of these techniques, see \citet{kaufmann2024survey}.

The efficacy of both RLHF and DPO is critically dependent on the quality of the underlying preference data, which is often noisy and inconsistent. Recent studies confirm that such data imperfections significantly degrade alignment performance \citep{gao2024impact, wang2024secrets}. This motivates the development of more robust methods. To address this, recent extensions of DPO have introduced mechanisms for robustness and calibration, including robust DPO \citep{chowdhury2024provably}, new robust loss functions \citep{liang2024robust}, and distributionally robust variants. While these methods mitigate the effects of noise during the training process, our work is complementary: we focus on improving data quality before training through enhanced annotation protocols and data filtering.

\subsection{Quality Control in Preference Annotation}
Ensuring the quality of preference annotations is typically managed through two complementary activities: estimation and improvement. Reliability is commonly estimated using traditional metrics like inter-annotator agreement \citep{krippendorff1989content, krippendorff2004reliability, artstein2008inter, monarch2021human}, known-answer or ``gold standard'' tasks \citep{callison2010creating}, and comparative contracts among annotators \citep{miller2005eliciting, bacon2012predicting, cai2015optimum, dasgupta2013crowdsourced}. More recent approaches leverage AI-assisted methods for real-time monitoring and error detection \citep{li2024labelaid, pustejovsky2012natural, qian2021annotation, northcutt2021confident, klie2024efficient, ghosal2022cicero}. To standardize these efforts, benchmark frameworks like AQUA \citep{goswami2023aqua} provide protocols for evaluating and enhancing annotation quality. When annotations fall below quality thresholds, common strategies include retraining annotators \citep{bareket2021neural, klie2024analyzing, ghosal2022cicero} or filtering out low-quality data points \citep{bastan2020author}. For broader overviews, see the surveys by \citet{daniel2018quality} and \citet{klie2024analyzing}.

A fundamental challenge in this area is the inherent subjectivity of preference data: ground-truth labels are often nonexistent, and the direct impact of annotation quality on downstream task performance can be difficult to measure. Our work confronts these challenges directly. We introduce a novel framework that leverages ordinary LLM users as annotators and integrates a robust quality control mechanism into the annotation process itself. We then empirically validate our framework by demonstrating its positive impact on downstream DPO tasks.

\subsection{Annotation Mechanism Design}
This work contributes to the field of annotation mechanism design, a formal framework for incentivizing self-interested agents (e.g., annotators) to act in alignment with a principal's objectives, particularly under conditions of information asymmetry. In the context of LLM alignment, recent applications of this framework include online learning strategies for RLHF with strategic annotators \citep{hao2024online}, algorithmic contracts to enhance text generation \citep{saig2024incentivizing}, incentive-compatible auctions for aggregating model outputs \citep{duetting2024mechanism}, Bayesian persuasion for generative AI \citep{harris2023algorithmic}, and collective mechanisms for fine-tuning with multiple reward models \citep{sun2024mechanism}.

Our approach diverges from this existing body of work, which typically assumes a pool of contracted, professional annotators. We instead focus on collecting feedback from ordinary, anonymous users—a setting where traditional quality control measures like contracts, performance tests, or "gold standard" questions are infeasible. To address this challenge, we introduce a novel quality signaling mechanism. We deploy two models with a known performance gap to generate responses to the same prompt; a user's preference between these outputs serves as a direct signal of their annotation quality. While the concept of leveraging model performance gaps has been explored in other contexts, such as AI-based evaluation \citep{zheng2023judging, burns2023weak, gao2024aligning}, our work is the first to formalize and apply it as a mechanism for filtering and weighting preference data from unvetted users to advance LLM alignment.

\subsection{The Expectation-Maximization (EM) Algorithm}
To estimate the parameters of our user behavior model from observed data, we employ the Expectation-Maximization (EM) algorithm. The EM algorithm is an iterative method for finding maximum likelihood estimates in models with latent variables. It operates by alternating between an Expectation (E) step, which computes the expectation of the complete-data log-likelihood given the observed data and current parameter estimates, and a Maximization (M) step, which updates the parameters to maximize this expectation. The theoretical properties of the EM algorithm are well-established. Foundational work by \citet{wu1983convergence} established its key properties: monotonic convergence of the observed-data likelihood and, under mild regularity conditions, convergence to a stationary point. Subsequent research has analyzed its rate of convergence \citep{MENG1994413} and developed finite-sample guarantees \citep{balakrishnan2017statistical}.

\subsection{Crowdsourcing and Annotator reliability}
Our work is closely related to the classical literature on crowdsourcing, truth inference, and annotator reliability. \citet{DawidSkene} use EM to jointly infer latent true labels and annotator-specific error rates from redundant noisy annotations.
\citet{Liu_NIPS2012_cd00692c} develop Bayesian and variational inference methods for crowdsourcing with latent labels and worker reliabilities. In contrast, our comparison-mode LLM setting often lacks repeated labels for the same prompt-response pair, and identifiability comes from the induced asymmetry between two response-generating LLMs rather than from a shared worker-item label matrix.
\citet{karger2011budget} study budget-optimal crowdsourcing and low-rank or belief-propagation aggregation under controlled task-assignment graphs. 
our problem setting differs in that the platform cannot freely assign arbitrary user prompts to annotators or rely on designed redundancy.
\citet{zhang2016spectral} develop a spectral-plus-EM framework that provides provable algorithms for Dawid--Skene-type crowdsourcing models. In contrast, our theoretical analysis is based on a likelihood argument tailored to the proposed LLM comparison-mode behavioral model.
\citet{khetan2017learning} study learning from noisy singly labeled data by coupling worker-quality estimation with a supervised prediction model. Our setting is similarly low-redundancy, but its identification signal comes from the A/B LLM design rather than agreement with a learned classifier.

\section{Background on RLHF and DPO}
Reinforcement Learning from Human Feedback (RLHF) and Direct Preference Optimization (DPO) are two prominent methods for aligning large language models with human preferences. Both approaches leverage pairwise preference data, where humans indicate which of two model-generated responses they prefer for a given prompt. The foundation for modeling such preferences is often the Bradley-Terry model.

\subsection{The Bradley-Terry Model}
The Bradley-Terry model \citep{bradley1952rank} is a statistical model for predicting the outcome of a paired comparison. Given two items, $i$ and $j$, the model assumes that each item has an intrinsic, latent ``strength'' or ``score'', let's say $r_i$ and $r_j$. The probability that item $i$ is preferred over item $j$ is given by:
\[
\mathbb{P}(i>j) =\dfrac{e^{r_i}}{e^{r_i} + e^{r_j}} = \sigma(r_i - r_j),
\]
where $\sigma(x) = 1/(1+ e^{-x})$ is the sigmoid function. This model provides a simple and effective way to convert underlying scores into pairwise probabilities, which is the cornerstone of training preference-based models.

\subsection{Reinforcement Learning from Human Feedback}
RLHF is a multi-stage process designed to fine-tune a language model to better align with human expectations. It typically starts from a base language model, which we'll call the reference policy $P_{\text{ref}}$.

\begin{enumerate}
    \item \textbf{Reward Model Training:} A separate model, the reward model (RM), is trained to predict human preferences. The training data $\mathcal{D}$ takes the form introduced in Section \ref{sec:RM_Modeling},
    \[
    \mathcal{D} = \{(x_i, y_{i1}, y_{i2}, z_i)\}_{i=1}^n.
    \]
    The reward model $r_{\phi}(x,y)$ (parameterized by $\phi$) is trained to assign a higher score to the preferred response. Using the Bradley-Terry model, the loss function for the reward model is the negative log-likelihood of the human preferences:
    \[
    L_{\text{RM}}(\phi) = -\sum_{i=1}^n \log \sigma \Big( z_i'\cdot(r_\phi(x_i, y_{i1}) - r_\phi(x_i, y_{i2})) \Big),
    \]
    where $z_i' = 2 z_i - 1$.
    \item \textbf{RL Fine-Tuning:} The policy $P_\theta$ (initialized from $P_{\text{ref}}$) is fine-tuned using reinforcement learning. The learned reward model $r(x, y)$ serves as the reward signal. The objective is to learn a policy that produces high-reward outputs while not diverging too far from the original reference policy. This is framed as minimizing the following loss function:
    \begin{equation}\label{eq:cal_L_RLHF}
        \mathcal{L}_{\text{RLHF}}(\theta) = \mathbb{E}_{x \sim \mathcal{D}, y \sim P_\theta(y|x)} \left[ \beta \log \frac{P_\theta(y|x)}{P_{\text{ref}}(y|x)} - r(x, y) \right]
    \end{equation}
    where $\beta$ is a hyperparameter that controls the strength of the Kullback-Leibler (KL) divergence penalty. Minimizing this loss encourages the policy to generate responses that yield high rewards (by making the negative reward term smaller) while penalizing divergence from the reference policy. This optimization is typically performed using algorithms like Proximal Policy Optimization (PPO).    
\end{enumerate}

\subsection{Direct Preference Optimization}
DPO simplifies the RLHF pipeline by eliminating the need to explicitly train a reward model and then perform reinforcement learning. It derives a direct loss function for the policy based on the preference data. Notice the policy $P^*$ that minimizes \eqref{eq:cal_L_RLHF} enjoys an closed form:
\[
P^*(y|x) = \dfrac{1}{Z(x)} P_{\text{ref}}(y|x)\cdot\text{exp}\left(\dfrac{1}{\beta} r(x, y)\right),
\]
where $Z(x)$ is a partition function that normalizes the distribution. From this relationship, we can express the reward function in terms of the optimal policy $P^*$ and the reference policy $P_{\text{ref}}$:
\[
r(x,y) = \beta \log\left(\dfrac{P^*(y|x)}{P_{\text{ref}}(y|x)}\right) + \beta \log Z(x).
\]
This is a crucial step: it shows that the reward function is implicitly ``defined'' by the optimal and reference policies. Now, we can plug this into the reward model loss function and derive the loss function for DPO:
\[
L_{\text{DPO}}(\theta) \;=\; - \sum_{i=1}^n 
\log \sigma\Bigg(
  \beta\, z'_i \cdot \Big(
    \log \dfrac{P_\theta(y_{i1}|x_i)}{P_{\rm ref}(y_{i1}|x_i)}
    - 
    \log \dfrac{P_\theta(y_{i2}|x_i)}{P_{\rm ref}(y_{i2}|x_i)}
  \Big)
\Bigg).
\]

\section{Extended Discussions on Algorithm \ref{alg:EM_UBM}}\label{apx:extended_discussion_of_EM}

This section provides further discussions on the implementation of Algorithm \ref{alg:EM_UBM}, especially focusing on Examples \ref{examp:1} and \ref{examp:2} discussed in Section \ref{sec:para_estimate_EM}. 

\subsection{Two-point discrete distribution}\label{subapx:extend_two_point_discrete}
When $\mathcal{P}_\eta$ follows the two point discrete distribution described in Example 1
\[
\mathcal{P}_\eta = q_1 \delta_{\underline{\eta}} + q_2 \delta_{\bar{\eta}},
\]
the complete-data log-likelihood function $Q(\theta\mid \theta^{(t)})$ takes the following form:
\begin{equation}\label{eq:Q_two_pointmass}
\begin{aligned}
    Q(\theta\mid \theta^{(t)})= \sum_{j=1}^m \Big(\gamma_{j1}^{(t)} \log q_1 +  \gamma_{j2}^{(t)} \log q_2\Big)+ \sum_{j=1}^m \sum_{i=1}^{n_j} \Big(\gamma_{j1}^{(t)} l(z_{i}^{(j)}\mid \mu_j, \underline{\eta}) + \gamma_{j2}^{(t)} l(z_{i}^{(j)}\mid \mu_j, \bar{\eta})\Big)
\end{aligned}
\end{equation}
where the posterior probabilities of $\eta_j$'s are
\begin{equation}\label{eq:posterior_of_eta}
    \gamma_{j1}^{(t)} = \dfrac{q_1^{(t)} \exp\Big(\sum_{i=1}^{n_j} l(z_{i}^{(j)}\mid \mu_j^{(t)}, \underline{\eta}^{(t)})\Big)}{q_1^{(t)} \exp\Big(\sum_{i=1}^{n_j} l(z_{i}^{(j)}\mid \mu_j^{(t)}, \underline{\eta}^{(t)})\Big) + q_2^{(t)} \exp\Big(\sum_{i=1}^{n_j} l(z_{i}^{(j)}\mid \mu_j^{(t)}, \bar{\eta}^{(t)})\Big)},\quad \gamma_{j2}^{(t)} = 1- \gamma_{j1}^{(t)}
\end{equation}
Deriving the form of $Q(\theta\mid \theta^{(t)})$ in \eqref{eq:Q_two_pointmass} involves a reparameterization trick to resolve the issue of shifted supports: for different realizations of $\underline{\eta}$ and $\bar{\eta}$, the support of $\mathcal{P}_\eta$ changes. Specifically, we reparameterize the latent variable from the continuous attentiveness level $\eta_j$ to a binary latent variable $h_j$, where $h_j=1$ indicates that user $j$ is attentive and $h_j=0$ indicates that the user is casual. The posterior of $h_j$ given $\theta^{(t)}$ is
\[
\mathbb{P}(h_j = 0\mid \theta^{(t)}) = \gamma_{j1}^{(t)},\quad \mathbb{P}(h_j = 1\mid \theta^{(t)}) = \gamma_{j2}^{(t)},
\]
and the complete-data log-likelihood, as a function of $\theta$, is
\[
L_c(\theta) = \sum_{j=1}^m \Big(h_j\log q_2 + (1-h_j)\log q_1 \Big) + \sum_{j=1}^m \sum_{i=1}^{n_j} \Big( h_j \cdot l(z_{i}^{(j)}\mid \mu_j, \bar{\eta}) + (1-h_j) \cdot l(z_{i}^{(j)}\mid \mu_j, \underline{\eta})\Big).
\]
Take the posterior expectation of $L_c(\theta)$ gives $Q(\theta\mid \theta^{(t)}) = \mathbb{E}_{h_1,\ldots,h_m}[L_c(\theta)\mid \theta^{(t)}]$, corresponding to \eqref{eq:Q_two_pointmass}.

If Assumption \ref{ass:mu} holds, then
Step \ref{step:M-step} of Algorithm \ref{alg:EM_UBM} (the M-step) has a closed-form solution, given by:
\[
q_1^{(t+1)} = \dfrac{1}{m} \sum_{j=1}^m \gamma_{j1}^{(t)}
\]
\[
\underline{\eta}^{(t+1)} = \dfrac{\sum_{j=1}^m \left(2\sum_{i=1}^{n_j} z_{i}^{(j)} - n_j\right)\cdot \gamma_{j1}^{(t)}}{(2\mu-1)\cdot\sum_{j=1}^m n_j\cdot \gamma_{j1}^{(t)}},\,\, \text{clipped to } [0,1]
\]
\[
\bar{\eta}^{(t+1)} = \dfrac{\sum_{j=1}^m \left(2\sum_{i=1}^{n_j} z_{i}^{(j)} - n_j\right)\cdot \gamma_{j2}^{(t)}}{(2\mu-1)\cdot\sum_{j=1}^m n_j\cdot \gamma_{j2}^{(t)}},\,\, \text{clipped to } [0,1]
\]

\subsection{Beta distribution}\label{subapx:extend_beta}
When $\mathcal{P}_\eta$ follows the Beta distribution described in Example 2
\[
\mathcal{P}_\eta = \text{Beta}(\alpha, \beta),
\]
under Assumption \ref{ass:mu}, the $Q(\theta\mid \theta^{(t)})$ function takes the following form
\[
\begin{aligned}
    Q(\theta\mid \theta^{(t)}) \sim -m \log B(\alpha, \beta) + \sum_{j=1}^m \mathbb{E}_{\eta_j}\Big[ (\alpha-1)\log \eta_j + (\beta-1) \log (1-\eta_j)\mid \theta^{(t)}\Big]
\end{aligned}
\]
where ``$\sim$'' denotes equality after discarding terms constant in $\theta$. The expectations are taken with respect to the posterior distribution of $\eta_j$, given by
\[
\eta_j \mid \theta^{(t)} \propto \exp\Big(\sum_{i=1}^{n_j}l(z_{i}^{(j)}\mid \mu, \eta_j)\Big)\cdot \eta_j^{\alpha^{(t)}-1}(1-\eta_j)^{\beta^{(t)}-1}.
\]
Maximizing $Q(\theta \mid \theta^{(t)})$ with respect to $(\alpha, \beta)$ is equivalent to solving the following equation system:
\[
\begin{aligned}
\psi(\alpha) - \psi(\alpha+\beta) &= \frac{1}{m} \sum_{j=1}^{m}\mathbb{E}_{\eta_j}[\log \eta_j\mid \theta^{(t)}]\\
\psi(\beta) - \psi(\alpha+\beta) &= \frac{1}{m} \sum_{j=1}^m \mathbb{E}_{\eta_j}[\log (1-\eta_j)\mid \theta^{(t)}]\\
\end{aligned}
\]
where $\psi(x) := \frac{\Gamma'(x)}{\Gamma(x)}$. The expectation terms $\E_{\eta_j}[\cdot]$ are generally intractable in closed form but can be approximated using numerical integration or Monte Carlo sampling from the conditional posterior distribution of $\eta_j$. The resulting system of equations can then be solved efficiently using iterative methods such as Newton–Raphson or quasi-Newton updates.

\section{Implementation Details and Additional Experiments}
In this section, we provide implementation details for the experiments in Section \ref{sec:experiments}, along with some additional experiment results.

\subsection{Estimating $\mu$ on Different Datasets}\label{appdix:pref_different_ds}
We estimate the value of $\mu$ for various pairs of generation LLMs A and B across three widely used preference datasets, including \texttt{UltraFeedback} \citep{cui2023ultrafeedback}, \texttt{HelpSteer3} \citep{wang2025helpsteer3}, and \texttt{HH-RLHF} \citep{ganguli2022red}. The generation models considered include \texttt{Qwen2.5-7B-Instruct}, \texttt{Qwen2.5-0.5B-Instruct}, \texttt{Llama-3.1-Tulu-3-8B-SFT} and \texttt{Llama-3.2-1B}, abbreviated as ``Qwen 7b'', ``Qwen 0.5b'', ``Llama 8b'', and ``Llama 1b''. For each dataset, we estimate $\mu$ by (1) randomly sampling 20,000 prompts, (2) generating responses with models A and B respectively, (3) evaluating each prompt–response pair using the reward model, and (4) computing the proportion of samples for which model A’s score exceeds that of model B. Throughout the paper, we use \texttt{Skywork-Reward-Gemma-2-27B} as the reward model. The results are summarized in Table \ref{tab:mu_on_diff_dataset}. Across different datasets, the estimated values of $\mu$ remain largely consistent, particularly when the performance gap between the two models is substantial (e.g., Qwen 7B vs. Qwen 0.5B, Qwen 7B vs. Llama 1B, and Llama 8B vs. Qwen 0.5B), where $\mu$ consistently exceeds 0.9.

\begin{table}[t]
\centering
\setlength{\tabcolsep}{4pt}
\renewcommand{\arraystretch}{1.15}

\begin{minipage}{\columnwidth}
\centering
\caption*{UltraFeedback}
\vspace{-0.3cm}
\begin{tabular}{lcccc}
\toprule
\multicolumn{1}{c}{LLMs A \& B} & Qwen 7b & Llama 8b & Llama 1b & Qwen 0.5b \\
\midrule
Qwen 7b   & 0.5 &   0.74   &   0.90   &  0.98    \\
Llama 8b  &  0.26  & 0.5 &  0.75  &   0.92   \\
Llama 1b  &  0.10  &  0.25  &  0.5   &  0.79    \\
Qwen 0.5b & 0.02   &  0.18    &   0.21   & 0.5 \\
\bottomrule
\end{tabular}
\end{minipage}

\vspace{10pt}

\begin{minipage}{\columnwidth}
\centering
\caption*{HelpSteer3}
\vspace{-0.3cm}
\begin{tabular}{lcccc}
\toprule
\multicolumn{1}{c}{LLMs A \& B} & Qwen 7b & Llama 8b & Llama 1b & Qwen 0.5b \\
\midrule
Qwen 7b   & 0.5 &  0.72   &   0.93   &   0.98   \\
Llama 8b  &  0.28  & 0.5   &  0.80    &    0.96  \\
Llama 1b  & 0.07   &  0.20    &  0.5  & 0.79  \\
Qwen 0.5b &  0.02  &  0.04  &  0.21    &  0.5   \\
\bottomrule
\end{tabular}

\end{minipage}

\vspace{10pt}

\begin{minipage}{\columnwidth}
\centering
\caption*{HH}
\vspace{-0.3cm}
\begin{tabular}{lcccc}
\toprule
\multicolumn{1}{c}{LLMs A \& B} & Qwen 7b & Llama 8b & Llama 1b & Qwen 0.5b \\
\midrule
Qwen 7b   & 0.5 &  0.75    &  0.90    &  0.99    \\
Llama 8b  & 0.25   & 0.5   &   0.7   &   0.93   \\
Llama 1b  &  0.10  &  0.30    & 0.5   &   0.82   \\
Qwen 0.5b &   0.01 &  0.07    &   0.18   & 0.5   \\
\bottomrule
\end{tabular}
\end{minipage}
\vspace{6pt}
\caption{Pairwise comparisons of generation LLMs across three preference datasets. The rows correspond to model A, the columns to model B, and each entry reports the estimated preference probability $\mu$ that model A is preferred over B. Across all three datasets, the estimates are generally consistent, varying by no more than $\pm 0.05$.}
\label{tab:mu_on_diff_dataset}
\end{table}

\subsection{Additional Synthetic Experiments}\label{subapx:more_simulations}

In Section \ref{subsec:simulation}, we conduct synthetic experiments to show the asymptotic convergence of Algorithm \ref{alg:EM_UBM}. Unless stated otherwise, all the experiments set $\mu=0.8$, $m=400$, and $50\leq n_j\leq 100$ for $j=1,\ldots,m$. The implementation of Algorithm \ref{alg:EM_UBM} follows Assumption \ref{ass:mu} and sets $\mu$ to its true value.

A numerical analysis of the convergence is provided in Table \ref{tab:relative-uncertainty}. Each row corresponds to a specific combination of sample sizes $(m, n)$, where $n$ is the common value assigned to all $n_j, j=1,\ldots, m$. Each column corresponds to an estimated parameter $\hat{\theta}$. The $\Delta$ column gives the relative estimation error, defined as the maximum relative error $|\hat{\theta} - \theta| / |\theta|$ across all scalar parameters (e.g., $\Delta = \max\{|\hat{\alpha} - \alpha| / |\alpha|,|\hat{\beta} - \beta| / |\beta| \}$ in Example 2). With $m, n\to\infty$, the relative estimation error approaches $0$. When aiming to reduce estimation error, increasing $m$ and increasing $n$ have different impacts across settings: consider the row corresponding to $(m,n) = (200,50)$, for Example 1, doubling $m$ is more effective in reducing the error, whereas for Example 2, doubling $n$ yields greater improvement.

\begin{table}[htbp]
\centering
\setlength{\tabcolsep}{8pt}  
\renewcommand{\arraystretch}{1.2}  
\begin{tabular}{@{}c | cccc | ccc@{}}
\toprule
\multirow{2}{*}{\begin{tabular}{c}
Sample size \\ $(m,n)$
\end{tabular}} &
\multicolumn{4}{c|}{Two-point discrete} &
\multicolumn{3}{c}{Beta} \\
\cline{2-5}\cline{6-8}
 & $q_1$ & $\underline{\eta}$ & $\bar{\eta}$ & $\Delta\downarrow$ &
   $\alpha$ & $\beta$ & $\Delta\downarrow$ \\  
\midrule
$\theta^*$ &  0.6   &    0.4    &  0.98      &    $-$    &   3     &   5     &    $-$    \\
\midrule
(200, 50)   &    0.53       &    0.35      &    0.95    &   $12.0\%$     &  3.61      &   5.50     &     $20.3\%$   \\
\midrule
(400, 50) &    0.63   &       0.42   &        0.99    &    $5.1\%$  &    2.54    &   4.28     &    $15.1\%$    \\
(200, 100)&    0.61   &  0.39     &        0.99      &      $2.1\%$  &    2.41    &    4.05    &    $20.0\%$    \\
\midrule
(400, 100)&   0.60    &  0.40   &  0.97     &   $1.3\%$   &     2.89   &   4.73     &    $5.3\%$    \\
(800, 100)&   0.60 &   0.40 &    0.98 & $0.8\%$   &  2.85 &   4.77 &  $4.8\%$     \\
(800, 200)&   0.60 &  0.40      &       0.98     &   $0.5\%$  &   3.01     &  5.02      &     $0.4\%$   \\
\bottomrule
\end{tabular}
\vspace{0.2cm}
\caption{The relative estimation error of Algorithm \ref{alg:EM_UBM} under varying sample sizes. The true model parameters are denoted by $\theta^*$. In this experiment, all $n_j$ are fixed to a common value $n$. Each row reports the estimation results for a specific $(m,n)$ combination. Each column lists the estimated parameters, while the $\Delta$ column shows the relative estimation error. As $m,n \to \infty$, the relative error converges to zero.}
\label{tab:relative-uncertainty}
\end{table}

In Figure \ref{fig:model_misspecify}, we visualized the performance of Algorithm \ref{alg:EM_UBM} under model misspecification. The left panel assumes that the true $\mathcal{P}_\eta$ follows a mixture of two Beta distributions,
\[
\mathcal{P}_\eta = 0.6\cdot \text{Beta}(4, 16) + 0.4\cdot \text{Beta}(16, 4),
\]
and in the right panel the true $\mathcal{P}_\eta$ follows a logistic-normal distribution,
\[
\mathcal{P}_\eta = \sigma\Big(-0.6 + 0.8\cdot \mathcal{N}(0, 1)\Big)
\]
where $\sigma(\cdot)$ is the sigmoid function. When the sample size is large enough, we can implement Algorithm \ref{alg:EM_UBM} with a richer parametric family for $\mathcal{P}_\eta$. Consider $m = 4000$ and $n_j=500$ for $j=1,\ldots, m$, let the distribution family of the estimated $\mathcal{P}_\eta$ be a mixture of logistic-normal distributions:
\begin{equation}\label{eq:mix_logistic_normal}
    \mathcal{P}_\eta = \sigma\left(\sum_{k=1}^3 \pi_k\cdot \mathcal{N}(\mu_k, \sigma_k^2)\right),
\end{equation}
whose density is given by
\[
P_{\eta}(\eta) = \dfrac{1}{\eta (1-\eta)} \sum_{k=1}^3 \pi_k \dfrac{1}{\sqrt{2\pi}\sigma_k} \text{exp}\Bigg(-\dfrac{\Big(\log \frac{\eta}{1-\eta} - \mu_k\Big)^2}{2\sigma_k^2}\Bigg).
\]
The $Q(\theta\mid \theta^{(t)})$ function can be approximated in the following way:
\begin{equation}\label{eq:approximate_Q}
    \begin{aligned}
    Q(\theta \mid \theta^{(t)}) &\sim \sum_{j=1}^m \mathbb{E}_{\eta_j}\Bigg[ \log \Bigg(\sum_{k=1}^3 \pi_k \dfrac{1}{\sqrt{2\pi}\sigma_k} \text{exp}\Bigg(-\dfrac{\Big(\log \frac{\eta_j}{1-\eta_j} - \mu_k\Big)^2}{2\sigma_k^2}\Bigg)\Bigg)\,\Bigg|\, \theta^{(t)}\Bigg]\\
    &\approx \sum_{j=1}^m \log \Bigg(\sum_{k=1}^3 \pi_k \dfrac{1}{\sqrt{2\pi}\sigma_k} \text{exp}\Bigg(-\dfrac{\Big(\log \frac{\hat{\eta}_j}{1-\hat{\eta}_j} - \mu_k\Big)^2}{2\sigma_k^2}\Bigg)\Bigg)
\end{aligned}
\end{equation}
where the expectation with respect to $\eta_j\mid \theta^{(t)}$ is replaced by substituting an estimator $\hat{\eta}_j$, such as the posterior mean or maximum a posteriori(MAP) estimate. Maximizing an objective of the form \eqref{eq:approximate_Q} is equivalent to fitting a three-component Gaussian mixture to the transformed samples $\left\{\log \dfrac{\hat{\eta}_j}{1-\hat{\eta}_j}\right\}$, which can be efficiently solved using a standard EM algorithm (see Section 3.2 of \citet{SIG-034}). 

In Figure \ref{fig:appendix_d2}, we illustrate how the mixture of logistic-normal distribution family approximates the following true $\mathcal{P}_\eta$'s. From left to right, the true distributions $\mathcal{P}_\eta$ are: (1) $\mathcal{P}_\eta = 0.6\cdot \text{Beta}(4, 16) + 0.4\cdot \text{Beta}(16, 4)$, (2) $\mathcal{P}_\eta = \sigma\left(-0.6 + 0.8\cdot \mathcal{N}(0, 1)\right)$ , (3) $\mathcal{P}_\eta = 0.2\cdot \delta_{0.2} + 0.6\cdot \delta_{0.6} + 0.2\cdot \delta_{0.9}$, (4) $\mathcal{P}_\eta = 0.2\cdot \text{Beta}(8,32) + 0.6\cdot \text{Beta}(20,20) + 0.2\cdot \text{Beta}(32,8)$. In all these cases, the true $\mathcal{P}_\eta$ is approximated reasonably well. 

\begin{figure*}[t]
    \centering
    \begin{subfigure}[b]{0.24\columnwidth}
        \centering
        \includegraphics[width=\textwidth]{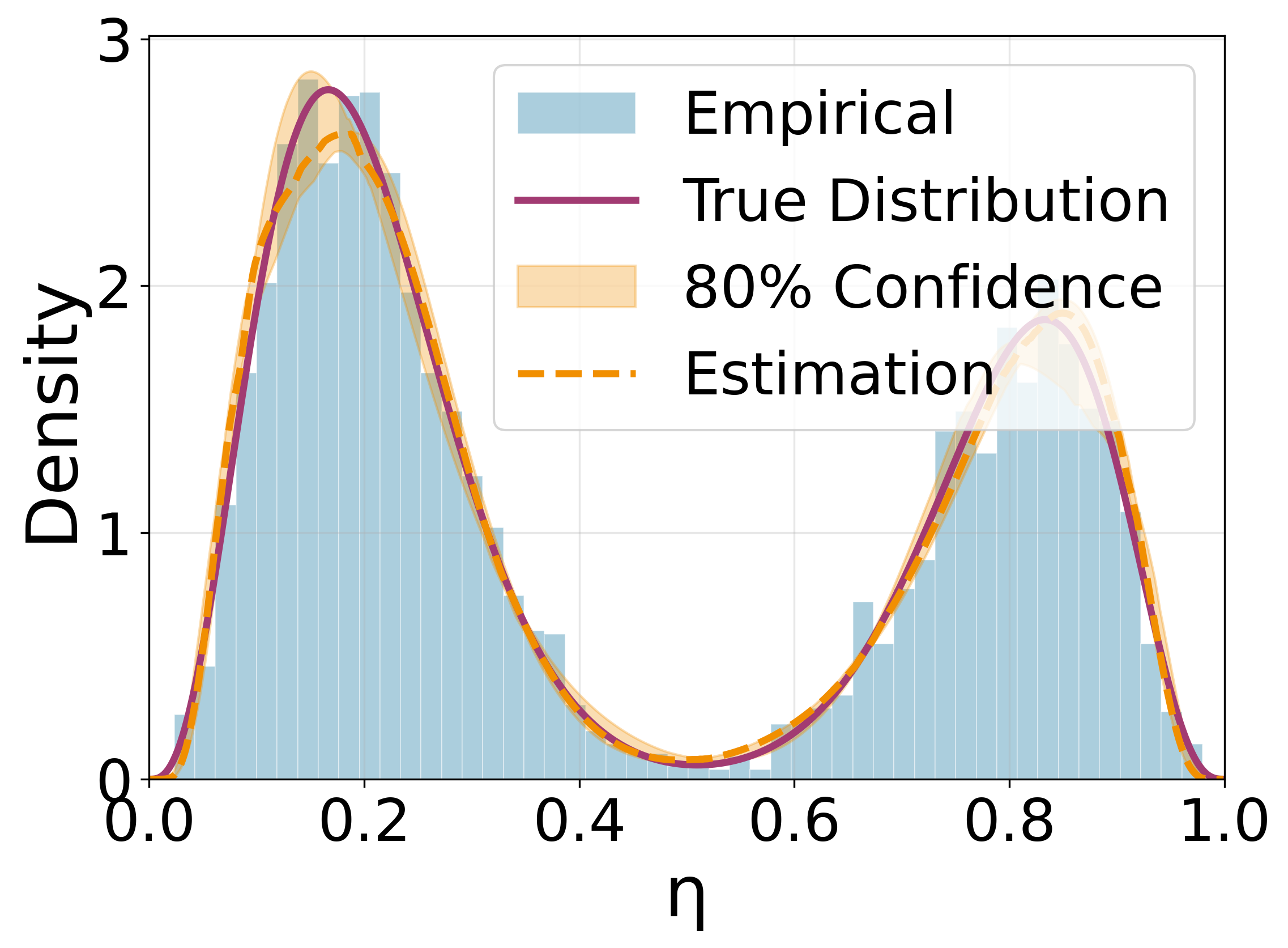}
    \end{subfigure}
    \hfill
    \begin{subfigure}[b]{0.24\columnwidth}
        \centering
        \includegraphics[width=\textwidth]{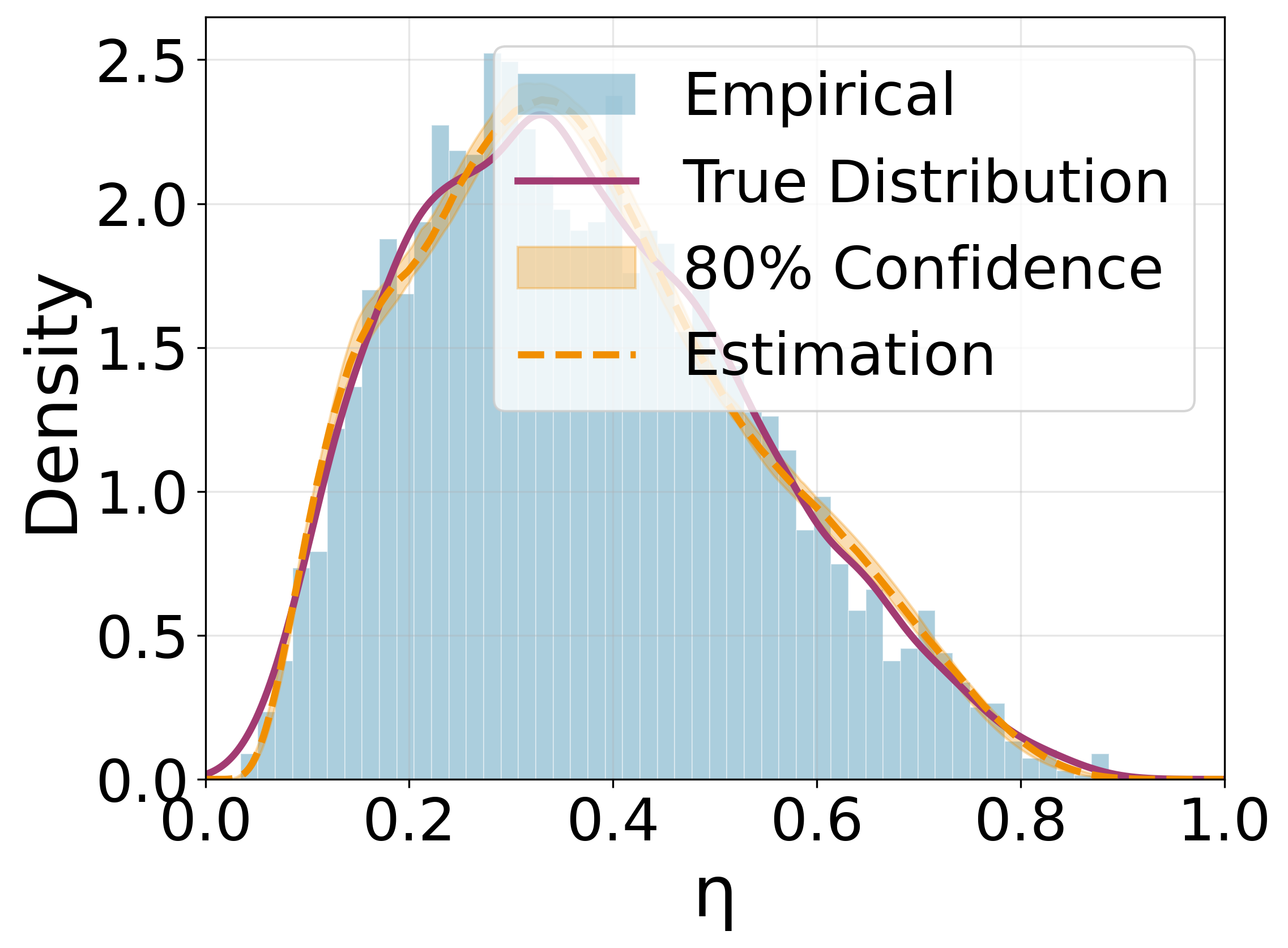}
    \end{subfigure}
    \hfill
    \begin{subfigure}[b]{0.24\columnwidth}
        \centering
        \includegraphics[width=\textwidth]{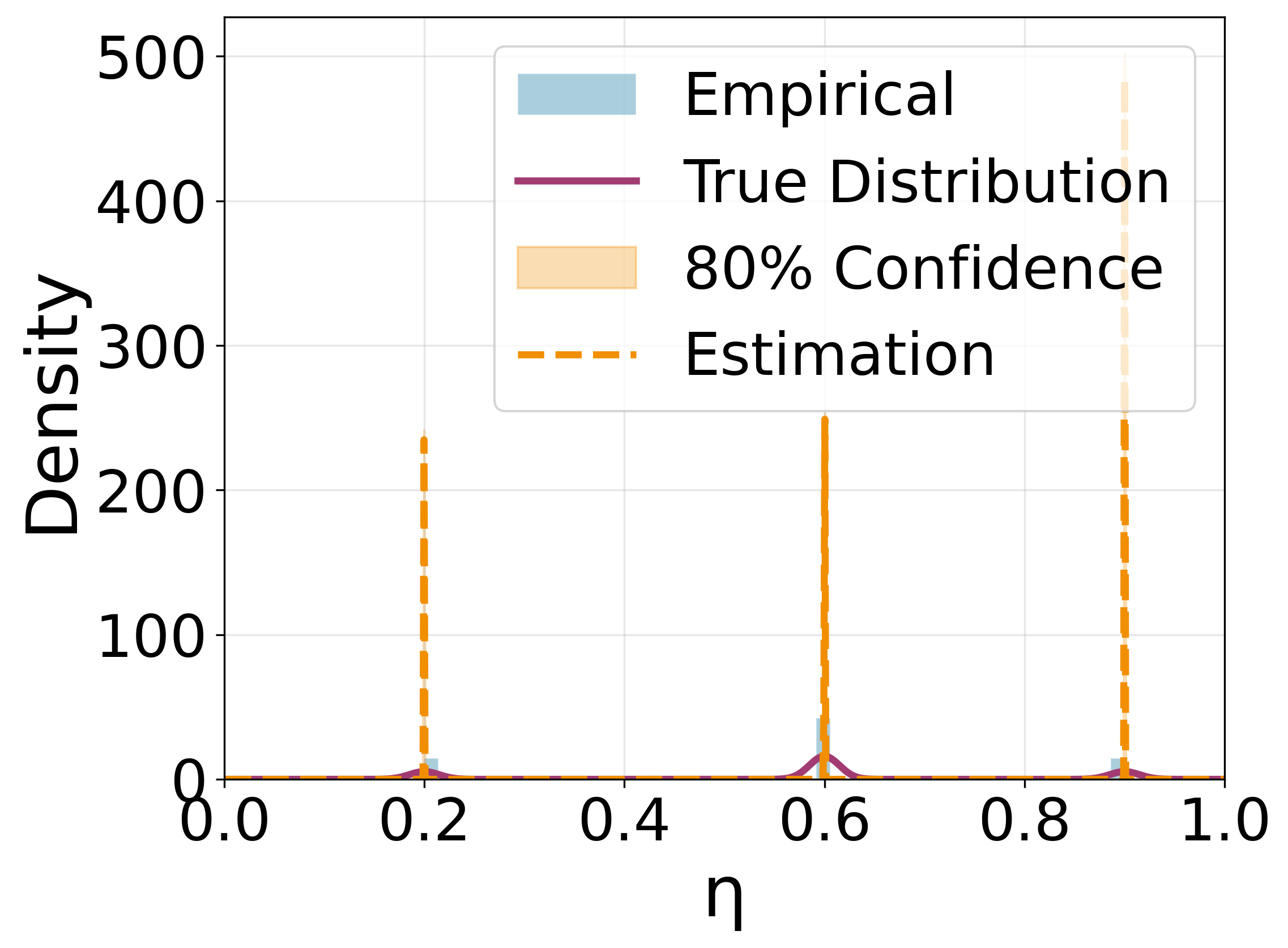}
    \end{subfigure}
    \hfill
    \begin{subfigure}[b]{0.24\columnwidth}
        \centering
        \includegraphics[width=\textwidth]{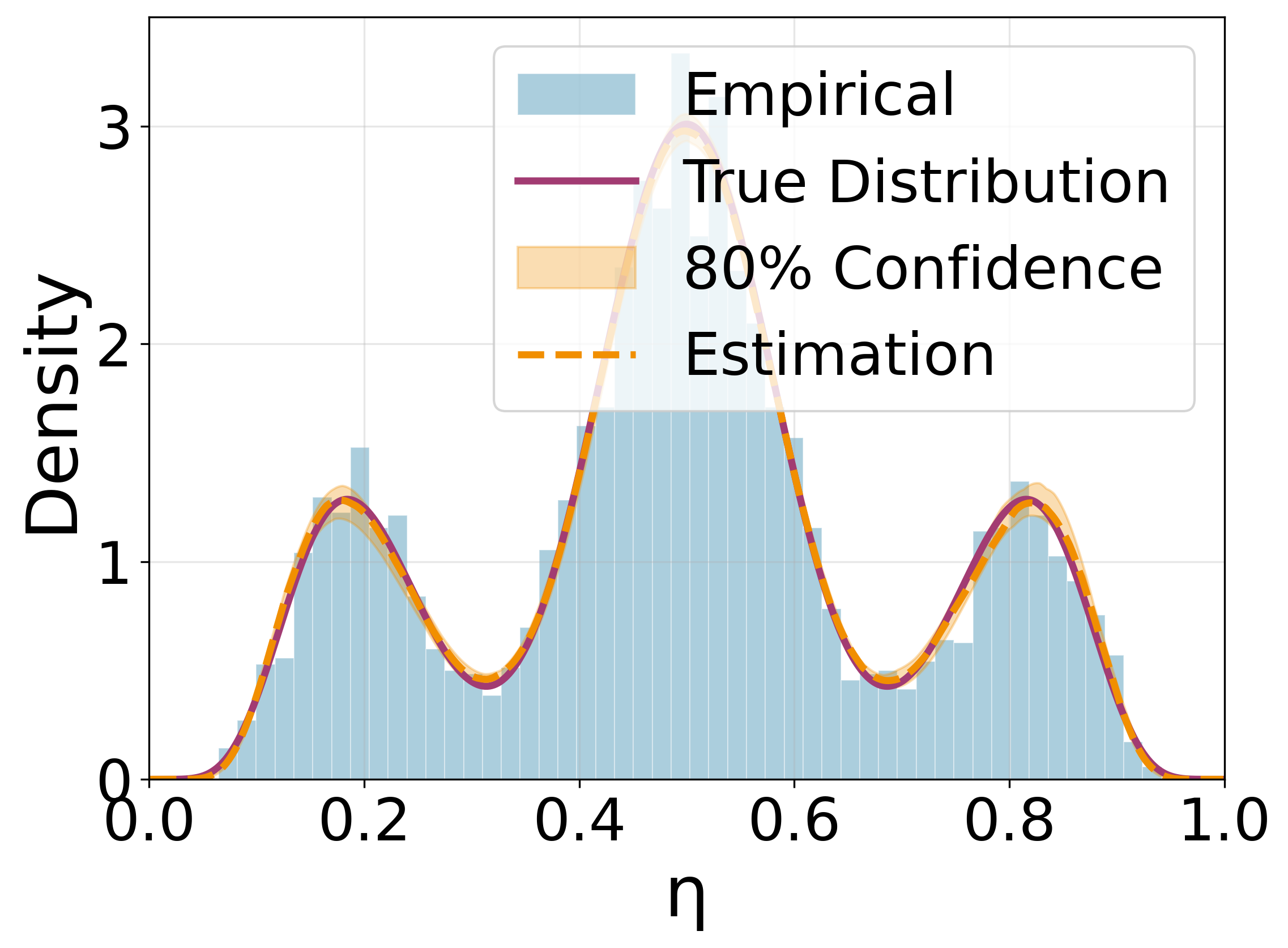}
    \end{subfigure}
    \caption{Approximating various true distributions $\mathcal{P}_\eta$ using a mixture of three logistic-normal distributions. The true distribution (the red line) is approximated by the estimated one (the orange line). In all cases, the mixture of logistic-normals provides a reasonable approximation to the true distribution.}
    \label{fig:appendix_d2}
\end{figure*}

\subsection{Relaxing Assumption \ref{ass:mu}: Unknown $\mu$}\label{sec:mu}
Throughout the paper, we have assumed that Assumption \ref{ass:mu} holds. This assumption indeed simplifies the analytical structure of Algorithm \ref{alg:EM_UBM}, as demonstrated in Appendix \ref{apx:extended_discussion_of_EM} and \ref{subapx:more_simulations}. It is also practically reasonable for data filtering in LLM alignment, as estimated in Appendix \ref{appdix:pref_different_ds}. Nevertheless, there may exist cases where this assumption does not hold. In this section, we investigate scenarios in which $\mu$ is not fixed a priori but must instead be treated as an unknown model parameter to be estimated within Algorithm \ref{alg:EM_UBM}.

We construct a synthetic experimental pipeline that consists of two stages: parameter estimation and data filtering. The experimental results are presented in Figure~\ref{fig:effect_of_mu}. Before discussing the figure, we first introduce the two stages of the pipeline.
\begin{itemize}
\item \textbf{Parameter estimation.} We consider two approaches. The first runs Algorithm \ref{alg:EM_UBM} with $\mu$ fixed at its true value, corresponding to the implementation adopted throughout this paper. The second specifies a fixed prior for $\mu$ via the regularization term $\mathcal{R}(\theta)$ of Algorithm \ref{alg:EM_UBM}. Specifically, we set the prior distribution of $\mu$ to $\text{Beta}(8,2)$ and define
\[
\mathcal{R}(\theta)=7\cdot \log \mu + \log (1-\mu).
\]
In Figure \ref{fig:effect_of_mu}, the two methods are termed ``Knowing $\mu$'' and ``Prior''.
\item \textbf{Data filtering.} Based on the estimated parameters, we identify and filter out the attentive users. For each user $j$, an estimated attentiveness level $\hat{\eta}_j$ is obtained via the maximum a posteriori (MAP) estimate of $\eta_j$. Two filtering strategies are considered. The first method selects the top $50\%$ of users ranked by $\hat{\eta}_j$, and the second method selects users whose estimated attentiveness exceeds the $0.5$ quantile of the true $\mathcal{P}_\eta$. In Figure \ref{fig:effect_of_mu}, these two strategies are denoted by ``(Ranking)'' and ``(Threshold)'', respectively.
\end{itemize}
Our experiment evaluates the overall performance of the synthetic pipeline by measuring its accuracy in recovering the truly attentive users, that is, the proportion of users whose true $\eta_j$ values lie above the $0.5$ quantile of $\mathcal{P}_\eta$ and are successfully filtered out. The evaluation is conducted under four different levels of $\mu$, with $\mathcal{P}_\eta$ fixed to $\text{Beta}(3,5)$. Figure \ref{fig:effect_of_mu} shows that the ability to identify attentive users is largely unaffected by whether $\mu$ is known exactly. For each fixed level of $\mu$, the three considered pipeline variants perform comparably. However, the performance is strongly influenced by the true value of $\mu$: as $\mu$ increases from $0.6$ to $0.9$, the accuracy of identifying attentive users consistently improves. This pattern is intuitive: when preferences between two responses are less distinct (corresponding to a smaller $\mu$), it becomes harder to discern whether a user’s choice reflects casual behavior or genuine ambiguity in the responses themselves.


\begin{figure}[ht!]
    \centering
    \includegraphics[scale=0.23]{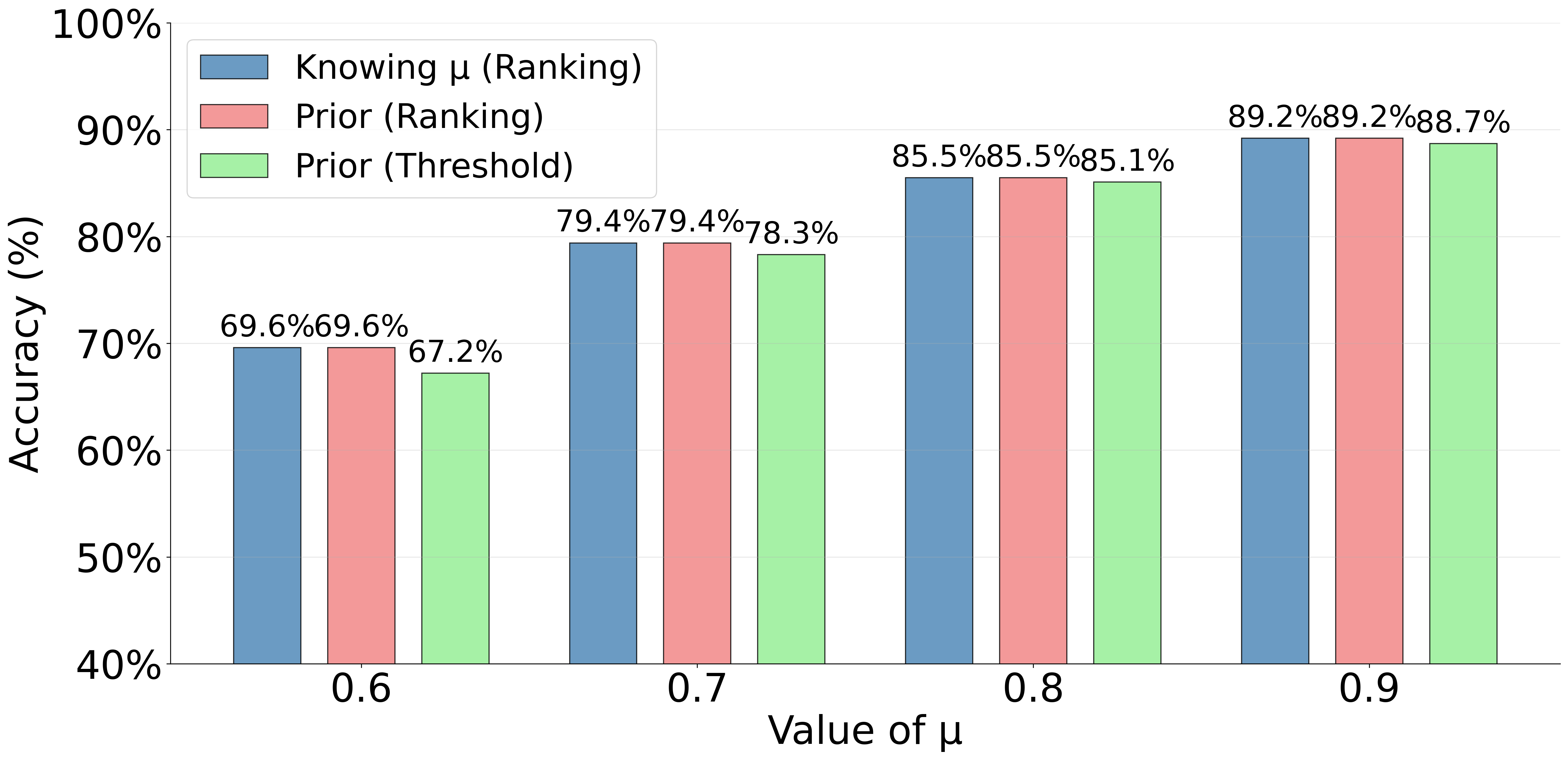}
    \caption{Accuracy in recovering the attentive users. Data are generated under $\mathcal{P}_\eta = \text{Beta}(3,5)$ with different values of $\mu$ as shown on the x-axis. In the legend, ``Knowing $\mu$'' means running Algorithm \ref{alg:EM_UBM} with $\mu$ set to its true value, ``Prior'' denotes running the algorithm with a fixed prior for $\mu$. Labels in brackets indicate the selection method: ``(Ranking)'' selects the top $50\%$ of users based on the posterior ranking of attentiveness, and ``(Threshold)'' selects users who exceed a certain threshold. The reported percentages show the proportion of attentive users correctly identified. Overall, the selection accuracy is only mildly affected by whether $\mu$ is fixed or estimated, but steadily increases with larger $\mu$.
    }\label{fig:effect_of_mu}

\end{figure}

\subsection{Relaxing Assumption \ref{ass:mu}: Unequal $\mu_j$}\label{subsec:uneq_mu}

Throughout most of the theoretical analysis and experiments, we assume that all users $j=1,\ldots,m$ share the same value of $\mu_j$, namely $\mu_1=\cdots=\mu_m=\mu$. This assumption is not overly restrictive when prompts are drawn from similar topics, that is, when users ask questions within the same general scope. Nevertheless, the proposed framework is not inherently limited to this setting, and in this section we discuss possible extensions that relax this assumption.

If $\mu_1,\ldots,\mu_m$ are allowed to vary freely without any constraints or correlation structure, then the system is generally unidentifiable. Consider the counterexample proposed in the main text. Suppose that the true parameters for user $j$ are $\mu_j=0.6$ and $\eta_j=1$, while an alternative parameterization sets $\mu_j=1$ and $\eta_j=0.2$. Both parameterizations imply the same distribution for each annotation $z_i^{(j)}$, making them statistically indistinguishable. Consequently, for the framework to be meaningful, additional constraints or assumptions are needed.

One simple approach is to cluster users into groups and assume a common value of $\mu_j$ within each group. This approach is applicable when users are drawn from different backgrounds and each user's conversations are concentrated around a fixed topic. One practical scenario is to select users who are PhD students in, for example, mathematics, chemistry, and biology, and collect conversation histories related to their respective research questions. In this setting, Algorithm~\ref{alg:EM_UBM} can be applied without any modification, and the global convergence guarantee carries over provided that the size of each group grows to infinity.

We present a synthetic experiment to demonstrate this idea. Let $m=5000$, and let $60\leq n_j\leq 100$ for $j=1,\ldots,m$. Each user $j$ belongs to Group A with probability $0.6$ and to Group B otherwise. For users in Group A, the model has $\mu_A=0.6$, and the user attentiveness is sampled from $\eta_j\sim \mathrm{Beta}(4,16)$. For users in Group B, the model has $\mu_B=0.9$, and the user attentiveness is sampled from $\eta_j\sim \mathrm{Beta}(16,4)$. In the generated data, $\mu_A$ and $\mu_B$ are unknown; only the group indicator (A or B) of each user is observed. Algorithm~\ref{alg:EM_UBM} assumes the mixture of logistic-normal  family $\mathcal{P}_\eta$ specified in~\eqref{eq:mix_logistic_normal}. The goal is to estimate: (i) $\mu_A$ and $\mu_B$, and (ii) the true distribution $\mathcal{P}\eta$. Figure~\ref{fig:unequal_mu_synthetic} below shows the trajectories of the estimated $\hat{\mu}_A$ and $\hat{\mu}_B$ over the EM iterations, together with the final estimate of $\mathcal{P}_\eta$. The EM algorithm moves $\hat{\mu}_A$ and $\hat{\mu}_B$ from their initialization $0.51$ towards their corresponding true values. The right panel shows that the estimated attentiveness distribution $\hat P_\eta$ captures the main bimodal structure of the true distribution $P_\eta$.

\begin{figure}[t]
    \centering

    \begin{minipage}{0.32\textwidth}
        \centering
        \includegraphics[width=\linewidth]{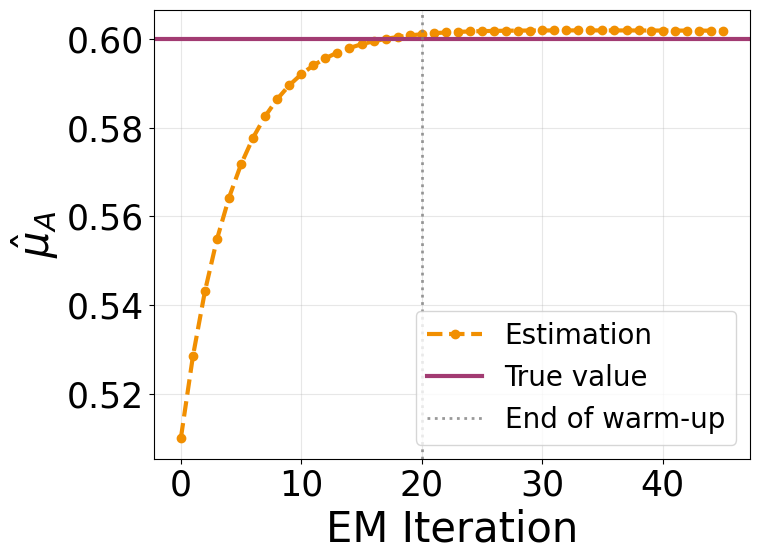}
        \vspace{-2mm}
        \caption*{\small (a) Trajectory of $\hat{\mu}_A$}
    \end{minipage}
    \hfill
    \begin{minipage}{0.32\textwidth}
        \centering
        \includegraphics[width=\linewidth]{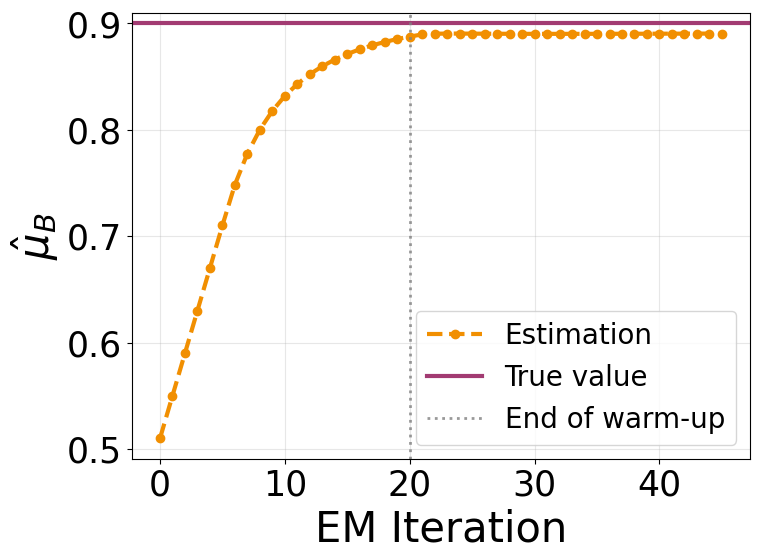}
        \vspace{-2mm}
        \caption*{\small (b) Trajectory of $\hat{\mu}_B$}
    \end{minipage}
    \hfill
    \begin{minipage}{0.32\textwidth}
        \centering
        \includegraphics[width=\linewidth]{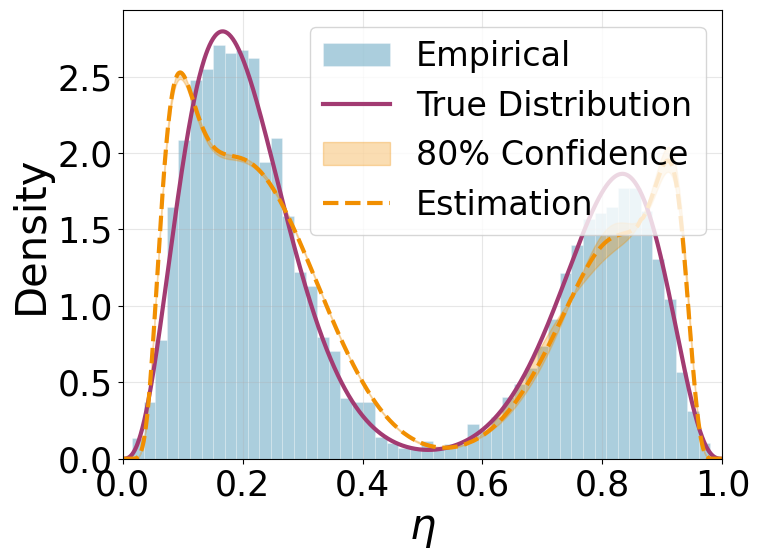}
        \vspace{-2mm}
        \caption*{\small (c) Estimation of $P_\eta$}
    \end{minipage}

    \caption{Synthetic experiment under unequal $\mu_j$'s. Users are divided into two groups with group-specific preference probabilities $\mu_A$ and $\mu_B$. The left and middle panels show the EM trajectories of the estimated $\hat{\mu}_A$ and $\hat{\mu}_B$, respectively. The right panel compares the final estimated attentiveness distribution $\hat{P}_\eta$ with the true distribution $P_\eta$.}
    \label{fig:unequal_mu_synthetic}
\end{figure}

Another option is to replace each $\mu_j$ with an estimate $\hat{\mu}_j$. A simple estimator can be constructed using a reward model $r(x,y)$. Specifically, for each user's feedback records $\mathcal{D}_j
=
\left\{
\left(x_i^{(j)},y_{i1}^{(j)},y_{i2}^{(j)},z_i^{(j)}\right)
\right\}_{i=1}^{n_j}$, we compute
\[
r_i^{(1)} = r(x_i^{(j)},y_{i1}^{(j)}),
\qquad
r_i^{(2)} = r(x_i^{(j)},y_{i2}^{(j)}),
\qquad
\hat{p}_i^{(j)} = \sigma\left(r_i^{(1)}-r_i^{(2)}\right),
\]
and define
\begin{equation}\label{eq:esti_hat_mu}
\hat{\mu}_j
=
\frac{1}{n_j}\sum_{i=1}^{n_j} \hat{p}_i^{(j)}.
\end{equation}
Algorithm~\ref{alg:EM_UBM} can then proceed by replacing each $\mu_j$ with its estimate $\hat{\mu}_j$. A potential concern with this approach is that, if a reward model $r(x,y)$ is already available, one may ask why user-labeled data are still needed, rather than using the reward model directly as the annotator. The reason is that the reward model only provides a noisy estimate $\hat{p}_i^{(j)}$ of the underlying true win probability $p(x_i^{(j)})$. Therefore, its induced annotations should not be treated as direct substitutes for user feedback in downstream training. However, when estimating $\mu_j$ through~\eqref{eq:esti_hat_mu}, the averaging operation can mitigate estimation errors, making $\hat{\mu}_j$ a reasonable estimator of $\mu_j$ under suitable calibration conditions.

A third method is to adopt a Bayesian approach. For each user $j$, we may specify a prior distribution $\pi_j$ over $\mu_j$ and run Algorithm~\ref{alg:EM_UBM} with the additional log-prior term
\[
\mathcal{R}(\theta) = C + \sum_{j=1}^m \log \pi_j(\mu_j).
\]
There are many possible ways to specify the prior distributions $\pi_j$. A desirable property is that the prior should reflect the uncertainty in the estimation of $\mu_j$. For example, suppose that users $j$ and $j'$ mainly ask two categories of questions, or prompts, denoted by A and B, respectively. If the LLMs or reward models are familiar with category A, then the estimate of $\mu_j$ may be accurate and have low uncertainty. In contrast, if they are less familiar with category B, then the estimate of $\mu_{j'}$ may be noisier. Such distinctions can be reflected through the choice of prior distributions. As an extreme example, for user $j$, one may use a point-mass prior
\[
\pi_j = \delta_{0.8},
\]
whereas for user $j'$, one may use a weakly informative prior supported on a wider interval, for example
\[
\pi_{j'}(\mu) \propto \mathbf{1}\{\mu\in[0.5,1]\}.
\]

\subsection{Additional Details on DPO Alignment}\label{subapx:more_DPO}
We present details for the DPO experiments in Section \ref{subsec:LLM_DPO_Align}. 

\paragraph{Models.} All the models are downloaded from the following sources:
\begin{itemize}
    \item \texttt{Skywork-Reward-Gemma-2-27B}: \url{https://huggingface.co/Skywork/Skywork-Reward-Gemma-2-27B-v0.2}
    \item \texttt{Qwen2.5-7B-Instruct}: \url{https://huggingface.co/Qwen/Qwen2.5-7B-Instruct}

    \item \texttt{Qwen2.5-7B}: \url{https://huggingface.co/Qwen/Qwen2.5-7B}
    
    \item \texttt{Qwen2.5-0.5B-Instruct}: \url{https://huggingface.co/Qwen/Qwen2.5-0.5B-Instruct} 
    
    \item \texttt{Llama-3.1-Tulu-3-8B-SFT}: \url{https://huggingface.co/allenai/Llama-3.1-Tulu-3-8B-SFT}

    \item \texttt{Llama-3.1-8B}: \url{https://huggingface.co/meta-llama/Llama-3.1-8B}

    \item \texttt{Llama-3.2-1B}: \url{https://huggingface.co/meta-llama/Llama-3.2-1B-Instruct}
    
\end{itemize}
Among the generative models, \texttt{Qwen2.5-7B} and \texttt{Llama-3.2-1B} are the base models for DPO training, while the remaining models are used to generate responses for constructing the training data.

\paragraph{Datasets.} The datasets are downloaded from the following sources:
\begin{itemize}
    \item \texttt{UltraFeedback}: \url{https://huggingface.co/HuggingFaceH4/ultrafeedback_binarized}

    \item \texttt{HelpSteer3}: \url{https://huggingface.co/nvidia/HelpSteer3}
\end{itemize}

\paragraph{Training.}
All experiments were conducted on two servers: one equipped with three NVIDIA A40 GPUs and the other with three NVIDIA L40 GPUs. For each DPO training task, we used approximately 20,000 samples, simulated from 400 synthetic users, each annotating between 30 and 50 samples. The training configuration is summarized as follows:
\begin{itemize}
\item Training epochs: 1
\item Effective batch size: 24
\item Learning rate: 5e-6
\item Learning rate schedule: cosine
\item Warm-up steps: 100
\item $\beta$ parameter for DPO loss: 0.1
\item LoRA rank ($r$): 16
\item LoRA $\alpha$: 32
\item LoRA dropout: 0.1
\end{itemize}

\paragraph{Evaluation.}
The evaluation metrics for the DPO-trained models include \textbf{MMLU}, \textbf{GSM8K}, the average \textbf{reward score} of the resulting models, and their \textbf{win rates} over the baseline models. These metrics together assess both factual knowledge, reasoning ability, and alignment quality of the DPO finetuned models. Specifically:
\begin{itemize}
\item \textbf{MMLU} \citep{hendrycks2020measuring}: Measures the model's performance on the massive multitask language understanding benchmark, which covers 57 subjects ranging from STEM and social sciences to humanities. It evaluates the model’s general knowledge and reasoning ability in a multiple-choice format.

\item \textbf{GSM8K} \citep{cobbe2021training}: Assesses the model’s arithmetic and logical reasoning on grade-school-level math problems. Each problem requires multi-step reasoning, and accuracy is measured by exact-match correctness of the final answer.

\item \textbf{Reward score}: The average scalar value assigned by the reward model to the model’s generated responses on a held-out set of test prompts. In our experiments, the test prompts are separated from the training data prior to DPO.

\item \textbf{Win rate}: The percentage of test prompts for which the trained model receives a higher reward score than its corresponding base model, when both are evaluated on the same test set.

\end{itemize}

\paragraph{Extended DPO experiments.}
Table \ref{tab:big_table_two_trials} presents an extended set of DPO experiments, expanding upon the results reported in Table \ref{tab:big_table} of the main text. When \texttt{Qwen2.5-7B} serves as the base model, we use \texttt{Qwen2.5-7B-Instruct} and \texttt{Qwen2.5-0.5B-Instruct} as the generative models A and B, respectively. When \texttt{Llama-3.1-8B} is the base model, we instead use \texttt{Llama-3.1-Tulu-3-8B-SFT} and \texttt{Llama-3.2-1B}. Overall, the results show that filtering for attentive users consistently enhances DPO performance, although the magnitude of improvement varies across different setups. For instance, when \texttt{Llama-3.1-8B} is trained on \texttt{UltraFeedback} with $\mathcal{P}_\eta$ Type $=2$, filtering yields only marginal gains, whereas in most other cases it leads to substantial performance improvements.

\begin{table}[h!]
\centering
{
\resizebox{\columnwidth}{!}{%
\begin{tabular}{P{1.4cm} | P{2.2cm} |P{1.5cm}p{1.6cm}P{1.5cm}P{1.5cm}P{1.5cm}P{1.5cm}}
\toprule
\textbf{Base} & \textbf{Dataset} & $\mathcal{P}_\eta$ \textbf{Type} & \textbf{Model} & \textbf{Score} & \textbf{Winrate} & \textbf{MMLU} & \textbf{GSM8K} \\
\midrule
\multirow{10}{*}{Qwen 7b} 
    & \multirow{5}{*}{UltraFeedback} 
        & - & Base & -5.47 & 50\% & 70.4 & 87.0 \\
        \cmidrule(lr){3-8}
    & & \multirow{2}{*}{1} & DPO & -3.59  & 63\% & 71.1 & 86.5 \\
    & &  & DPO+Filter  & -3.12 & 67\% & 72.0 & 89.0 \\
        \cmidrule(lr){3-8}
    & & \multirow{2}{*}{2} & DPO & -4.42  & 58\% & 71.3 & 87.0 \\
    & &  & DPO+Filter  & -3.75 & 64\% & 71.8 & 88.0 \\
\cmidrule(lr){2-8}
    & \multirow{5}{*}{HelpSteer3} 
        & - & Base & -8.39 & 50\% & 71.1 & 87.0 \\
        \cmidrule(lr){3-8}
    & & \multirow{2}{*}{1} & DPO & -6.56 & 67\% & 70.8 & 86.5 \\
    & &  & DPO+Filter & -5.89 & 73\% & 72.0 & 88.0 \\
        \cmidrule(lr){3-8}
    & & \multirow{2}{*}{2} & DPO & -6.91 & 67\% & 71.0 & 87.0 \\
    & &  & DPO+Filter & -6.83 & 68\% & 71.4 & 87.5 \\
\midrule
\multirow{10}{*}{Llama 8b} 
    & \multirow{5}{*}{UltraFeedback} 
        & - & Base &  -10.02 & 50\% & 61.7 & 51.5\\
        \cmidrule(lr){3-8}
    & & \multirow{2}{*}{1} & DPO & -9.51 & 52\% & 62.8 & 52.0\\
    & &  & DPO+Filter  &  -9.25 & 56\% & 63.1 & 51.5\\
        \cmidrule(lr){3-8}
    & & \multirow{2}{*}{2} & DPO &  -9.87 & 51\% & 63.4 & 53.0\\
    & &  & DPO+Filter  &  -9.7 & 52\% & 62.8 & 54.5\\
\cmidrule(lr){2-8}
    & \multirow{5}{*}{HelpSteer3} 
        & - & Base & -10.43 & 50\% & 61.4 & 53.0 \\
        \cmidrule(lr){3-8}
    & & \multirow{2}{*}{1} &  DPO & -9.27 & 58\% & 63.1 & 56.0\\
    & &  & DPO+Filter &  -9.06 & 66\% & 62.8 & 56.0\\
        \cmidrule(lr){3-8}
    & & \multirow{2}{*}{2} & DPO &  -9.74 & 56\% & 62.1 & 53.0\\
    & &  & DPO+Filter & -9.44 & 60\%  & 62.8 & 56.0\\
\bottomrule
\end{tabular}%
}
}
\vspace{0.2cm}
\caption{An extended version of Table \ref{tab:big_table}, presenting DPO alignment results for both Qwen 7B and Llama 8B. The overall trend holds that filtering for attentive users consistently enhances DPO performance. }
\label{tab:big_table_two_trials}
\end{table}

\section{Theoretical Discussions}

\subsection{Local Convergence of the EM Algorithm}\label{appdix:EM_local_cvg}

Let $\mathbf{z}$ denote the observed variables, $\boldsymbol{\eta}$ denote the latent variables, and $\theta$ the vector of model parameters. The joint distribution of $(\bz,\boe)$ is denoted by $p(\bz, \boe\mid \theta)$. In our setting, $\bz=\{z_i^{(j)}\}_{1\leq i\leq n_j,1\leq j\leq m}$, $\boe=\{\eta_j\}_{1\leq j\leq m}$, and $p(\bz, \boe\mid \theta)$ corresponds to the joint distribution induced by the user behavior model. The EM algorithm iteratively computes 
\[
Q(\theta\mid \theta^{(t)}) = \mathbb{E}_{\boe}\left[ \log p(\bz,\boe\mid \theta) \mid \bz, \theta^{(t)}\right]
\]
and
\[
\theta^{(t+1)} =\arg\max_{\theta\in \Theta} Q(\theta\mid \theta^{(t)}).
\]
Let
\[
L(\theta) := \log p(\bz\mid \theta) = \log \int_{\boe} p(\bz, \boe\mid \theta) \text{d} \boe
\]
denote the observed-data log-likelihood, the following result gives the convergence guarantee of the EM algorithm. 
\begin{lemma}
    The observed-data log-likelihood is guaranteed to be non-decreasing with the EM iterations, that is,
    \[
    L(\theta^{(t+1)}) \geq L(\theta^{(t)}).
    \]
\end{lemma}
\begin{proof}
    By Jensen's inequality,
    \[
    \begin{aligned}
        L(\theta) &= \log \int_{\boe} \dfrac{p(\bz,\boe\mid \theta)}{p(\boe\mid \bz, \theta^{(t)})} \cdot p(\boe\mid \bz, \theta^{(t)}) \text{d} \boe \\
        &\geq \int_{\boe} \Bigg(\log \dfrac{p(\bz,\boe\mid \theta)}{p(\boe\mid \bz, \theta^{(t)})}\Bigg) \cdot p(\boe\mid \bz, \theta^{(t)}) \text{d} \boe\\
        &= Q(\theta\mid \theta^{(t)}) - \int_{\boe} \log p(\boe\mid \bz, \theta^{(t)}) \cdot p(\boe\mid \bz, \theta^{(t)}) \text{d} \boe\\
        &= Q(\theta\mid \theta^{(t)}) - Q(\theta^{(t)}\mid \theta^{(t)}) + L(\theta^{(t)}).
    \end{aligned}
    \]
    For any $\theta$ such that $Q(\theta\mid \theta^{(t)}) \geq Q(\theta^{(t)}\mid \theta^{(t)})$, we have $L(\theta) \geq L(\theta^{(t)})$.
\end{proof}

Let $\{\theta^{(t)}\}$ be the EM sequence. \citet{wu1983convergence} establishes that, under mild conditions (continuity of $Q(\theta\mid \theta^{(t)})$ in both arguments and compactness of the parameter space $\Theta$), every limit point $\theta^\infty$ of $\{\theta^{(t)}\}_{t=1}^\infty$ is a stationary point of the observed-data log-likelihood, i.e.,
\[
\nabla_\theta L(\theta^\infty) = 0.
\]

\subsection{Uniform Law of Large Numbers}
Let $(z_1,\ldots,z_n)$ be i.i.d.\ random variables on a measurable space $(\mathcal{Z},\mathcal A)$
with common distribution $P$. Let $\Theta\subset\mathbb R^d$ be compact and let
\[
\mathcal F \;=\; \{ f_\theta:\mathcal Z\to\mathbb R^k \;:\; \theta\in\Theta \},
\qquad k\ge 1,
\]
be a function family. For $f:\mathcal Z\to\mathbb R^k$, write
\[
P_n f \;=\; \frac{1}{n}\sum_{i=1}^n f(z_i),
\qquad
P f \;=\; \mathbb E_P[f(z)].
\]

This section formalizes the notion of uniform convergence of empirical averages to their population expectations over the function class $\mathcal F$, known as the Glivenko–Cantelli property.
\begin{definition}[Glivenko--Cantelli]
The class $\mathcal F$ is \emph{Glivenko--Cantelli (GC)} under $P$ if
\[
\sup_{\theta\in\Theta}\bigl\| P_n f_\theta - P f_\theta \bigr\| \xrightarrow{p} 0,
\]
where $\|\cdot\|$ is any norm on $\mathbb R^k$.
\end{definition}

Intuitively, the GC property ensures that empirical averages approximate their expectations uniformly over all parameter values $\theta$. A sufficient condition for this property to hold is given by the following Lipschitz criterion.

\begin{lemma}[Lipschitz Criterion, Thm. 2.4.1 of \citet{vanderVaart1996}]\label{lem:ULLN-Lipschitz}
Suppose there exist measurable envelopes $H,G:\mathcal Z\to[0,\infty)$ such that
\begin{enumerate}
\item[(i)] (Envelope) $\|f_\theta(z)\|\le H(z)$ for all $\theta\in\Theta$, $z\in\mathcal Z$, with $\mathbb E[H(z)]<\infty$;
\item[(ii)] (Uniform Lipschitz continuity) for all $\theta,\theta'\in\Theta$,
\[
\| f_\theta(z) - f_{\theta'}(z) \| \;\le\; G(z)\,\|\theta-\theta'\|, \qquad \mathbb E[G(z)]<\infty.
\]
\end{enumerate}
Then $\mathcal F$ is GC:
\[
\sup_{\theta\in\Theta}\bigl\| P_n f_\theta - P f_\theta \bigr\| \xrightarrow{p} 0.
\]
\end{lemma}

Under stronger regularity conditions, one can further derive a finite-sample bound on the uniform deviation $\sup_{\theta\in \Theta} \| P_n f_\theta - P f_\theta\|$, as stated in the following lemma. We begin by recalling the definition of the \textit{covering number} and the notion of \textit{big-O in probability}.
\begin{definition}[Covering Number and Metric Entropy]
For a compact set $\Theta$, the covering number $N(\varepsilon,\Theta,\|\cdot\|)$ denotes the smallest integer $N$ such that $\Theta$ is covered by $N$ balls of radius $\varepsilon$ under $\|\cdot\|$. Its logarithm is the metric entropy. 

\end{definition}
\begin{definition}[Big-O In Probability]
    For a random process $\{X_n\}_{n=1}^\infty$ and a series $\{a_n\}_{n=1}^\infty$, we write 
    \[X_n = O_p(a_n)\] 
    to denote that $X_n$ is of order $a_n$ in probability, i.e., 
    \[
    \forall \varepsilon>0,\exists M<\infty\,\, \text{ such that }\,\, \mathbb{P}(|X_n|/ a_n >M)<\varepsilon\,\,\text{for all sufficiently large }n.
    \]
\end{definition}

\begin{lemma}[Thm. 2.14.1 of \citet{vanderVaart1996}]
\label{lem:finite-ULLN-Lipschitz-finite}
Let the compact set $\Theta$ have diameter $D:=\sup_{\theta,\theta'\in\Theta}\|\theta-\theta'\|<\infty$. Suppose the metric entropy satisfies
\begin{equation}\label{eq:bound_on_cvnumber}
    \log N(\varepsilon,\Theta,\|\cdot\|)\ \le\ v\log\!\Big(\frac{A}{\varepsilon}\Big)
\quad\text{for all }\varepsilon\in(0,1],
\end{equation}
for some $A\ge e$ and $v\ge 1$. Assume that $\mathcal{F}$ satisfies the same ``Envelope'' and ``Uniform Lipschitz continuity'' conditions as in Lemma \ref{lem:ULLN-Lipschitz}, and that there exist finite constants $B,L<\infty$ such that $H(z)\leq B$ and $G(z)\leq L$ for almost every $z$. Then there exists a universal constant $C>0$ such that for every $\delta\in(0,1)$, with probability at least $1-\delta$,
\[
\sup_{\theta\in\Theta}\bigl\|(P_n-P)f_\theta\bigr\|
\ \le\
C\left\{
B\,\sqrt{\frac{v\log(An)+\log(1/\delta)}{n}}
\;+\;
L\,D\,\sqrt{\frac{v}{n}}
\right\}.
\]
Put simply, we have \[\sup_{\theta\in\Theta}\bigl\|(P_n-P)f_\theta\bigr\| = O_p(\sqrt{\log n / n}).
\]
\end{lemma}
\begin{remark}\label{remark:remark1}
    The metric-entropy bound in \eqref{eq:bound_on_cvnumber} is generally mild and easy to satisfy in most practical settings. For any compact subset $\Theta\subset\mathbb{R}^d$ under a Euclidean (or equivalent) norm, the covering number scales polynomially with $1/\varepsilon$, i.e., there exists constant $C$ such that
\[
\log N(\varepsilon,\Theta,\|\cdot\|) \leq d\log \Big(\dfrac{D}{\varepsilon}\Big) + C.
\]
\end{remark}

A directly application of Lemma \ref{lem:finite-ULLN-Lipschitz-finite} is to characterize the asymptotic behavior of the empirical log-likelihood term $\sum_{i=1}^n l(z_i\mid \mu, \eta)$, which appears in equation \eqref{eq:likelihood_formula_L} and in Algorithm \ref{alg:EM_UBM}.

\begin{corollary}\label{cor:corollary_1}
    Let $(z_1,\ldots, z_n)$ be i.i.d. samples from $\text{Bernoulli}(p)$. $1/2<\mu<1$ is a known constant. Then
    \[
    \sup_{\eta\in [0, 1]} \Big|\dfrac{\sum_{i=1}^n l(z_i|\mu, \eta)}{n} - \Big(p \log (1/2 + \eta\cdot (\mu-1/2)) + (1-p)\log (1/2 - \eta\cdot (\mu-1/2)) \Big)\Big| = O_p(\sqrt{\log n / n}).
    \]    
\end{corollary}
\begin{proof}
    Recall the form of $l(z|\mu, \eta)$,
    \[
    l\left(z | \mu, \eta\right) :=z\cdot \log (1/2 + \eta\cdot (\mu - 1/2)) + (1-z)\cdot \log (1/2 - \eta\cdot (\mu - 1/2)).
    \]
    Define the function family $\mathcal{F}$ as
    \[
    \mathcal{F} := \{l(\cdot\mid \mu ,\eta): \eta\in \Theta_\eta = [0, 1]\}.
    \]    
    By Lemma \ref{lem:finite-ULLN-Lipschitz-finite} and Remark \ref{remark:remark1}, it remains to verify the ``Envelope'' and ``Uniform Lipschitz continuity'' conditions. The following observations conclude the proof.
    \begin{itemize}
        \item[(i)] (Envelope) For all $\eta\in [0,1]$ and all $z\in \{0, 1\}$, we have 
        \[
        |l(z\mid \mu,\eta)| \leq -\log (1 - \mu) <\infty.
        \]
        \item[(ii)] (Uniform Lipschitz continuity) For all $\eta\in [0,1]$ and all $z\in \{0, 1\}$, we have
        \[
        \begin{aligned}
            \Big| \dfrac{\partial}{\partial \eta} l(z\mid \mu,\eta) \Big| &= (\mu-1/2)\cdot \Big| \dfrac{z}{1/2 + \eta\cdot (\mu - 1/2)} + \dfrac{z-1}{1/2 - \eta\cdot (\mu - 1/2)}\Bigg| \\
            &= (\mu-1/2)\cdot \Bigg|\dfrac{z -1/2 - \eta\cdot (\mu-1/2)}{1/4 - \eta^2 (\mu-1/2)^2}\Bigg|\\
            &\leq \dfrac{\mu-1/2}{1-\mu} <\infty.
        \end{aligned}
        \]
    \end{itemize}
\end{proof}

\subsection{Proof of Theorem \ref{thm:simple_two_mass_converge}}\label{subsec:prove_theorem_two_mass}
The proof proceeds in two main steps. In the first step, we construct the population version of the observed-data log-likelihood and show that its stationary points are guaranteed to stay close to the true parameters. In the second step, we establish that the gradient of the sample log-likelihood $\nabla_\theta L(\theta)$ converges uniformly to its population counterpart. Along the way, we specify the regularity conditions required for $\mathcal{P}_\eta$ and verify that both Example \ref{examp:1} and \ref{examp:2} satisfy these conditions. Throughout the proof, we assume $\mu$ is known, and that $1/2 < \mu < 1$. We also impose the simplifying condition that $n_1 = \cdots = n_m = n$ for some $n$. This assumption is made purely for notational convenience; in general, as long as
\[
\sum_{j=1}^m n_j\to \infty,\quad \max_{j=1,\ldots, m} \dfrac{n_j}{\sum_{j'=1}^{m} n_{j'}}\to 0,
\]
the conclusions of Lemmas \ref{lem:ULLN-Lipschitz} and \ref{lem:finite-ULLN-Lipschitz-finite}, as well as all subsequent results, continue to hold without modification.

We will present the proof in two parts. The first part addresses the case where $\mathcal{P}_\eta$ has a continuous density function, which covers Example \ref{examp:2}. The second part provides a separate proof for Example \ref{examp:1}, where $\mathcal{P}_\eta$ is discrete.

\subsubsection{$\mathcal{P}_\eta$ with Continuous Density}\label{subsubsec:continuous_eta}
Let $p(\eta\mid \theta)$ denote the density function of $\mathcal{P}_\eta$. Then
\[
L(\theta) = \sum_{j=1}^m \log \Bigg( \int_{\eta_j\in[0,1]} \exp\Big(\sum_{i=1}^{n} l\left(z_{i}^{(j)}\big\vert\mu, \eta_j\right)\Big) p(\eta_j\mid \theta) \text{d} \eta_j \Bigg).
\]
Define the population version of $L(\theta)$ as follows,
\[
L^*(\theta) = m \int_{\eta^*\in [0,1]} \log \Bigg( \int_{\eta\in[0,1]} \exp\Big(n\cdot  \bar{l}\left(\eta\mid \eta^*\right)\Big) p(\eta\mid \theta)  \text{d}\eta \Bigg) p(\eta^*\mid \theta^*)\text{d} \eta^*,
\]
where $\theta^*$ are the true model parameters and
\begin{equation}\label{eq:def_l_bar}
    \bar{l}(\eta\mid \eta^*) := (1/2 + \eta^*\cdot (\mu-1/2)) \log (1/2 + \eta\cdot (\mu-1/2)) + (1/2 - \eta^*\cdot (\mu-1/2))\log (1/2 - \eta\cdot (\mu-1/2)).
\end{equation}
The gradients of $L(\theta)$ and $L^*(\theta)$ are:

\begin{equation}\label{eq:form_of_grad_L}
    \nabla_\theta L(\theta) = \sum_{j=1}^m \dfrac{\int_{\eta_j\in [0,1]} \exp\Big(\sum_{i=1}^{n} l\left(z_{i}^{(j)}\big\vert\mu, \eta_j\right)\Big) \nabla_\theta p(\eta_j\mid \theta) \text{d} \eta_j}{\int_{\eta_j\in [0,1]} \exp\Big(\sum_{i=1}^{n} l\left(z_{i}^{(j)}\big\vert\mu, \eta_j\right)\Big) p(\eta_j\mid \theta) \text{d} \eta_j}\\
\end{equation}
\begin{equation}\label{eq:form_of_grad_Lstar}
    \nabla_\theta L^*(\theta) =m\cdot \int_{\eta^*\in [0,1]} \Bigg(\dfrac{\int_{\eta\in [0,1]} \exp\Big(n\cdot  \bar{l}\left(\eta\mid \eta^*\right)\Big) \nabla_\theta p(\eta\mid \theta) \text{d} \eta}{\int_{\eta\in [0,1]} \exp\Big(n\cdot  \bar{l}\left(\eta\mid \eta^*\right)\Big) p(\eta\mid \theta) \text{d} \eta} \Bigg)\, p(\eta^*\mid \theta^*)\text{d} \eta^*
\end{equation}

We assume the following regularity conditions:
\begin{assumption}\label{assump:regular_p}
The model parameter $\theta$ is constrained to a compact and convex set $\Theta$, and the density function $p(\eta\mid \theta)$ satisfies the following conditions:
    \begin{enumerate}[label=(\alph*)]\item\label{itm:C2Continuous}$p(\eta\mid \theta)$ is $C^2$ continuous in $\theta$, and is supported on $(0,1)$.
        \item\label{itm:identify} If for some $\theta$ and $\theta'$, there is $p(\cdot\mid \theta)\equiv p(\eta\mid \theta')$, then $\theta=\theta'$.
        \item\label{itm:KLconvex} The Kullback–Leibler (KL) divergence between $p(\eta\mid \theta^*)$ and $p(\eta\mid \theta)$ is strongly convex in $\theta$. Concretely, define
        \[
        K(\theta) := \int_{\eta\in [0,1]}p(\eta\mid \theta^*)\log \dfrac{p(\eta\mid \theta^*)}{p(\eta\mid \theta)}  \text{d} \eta,
        \]
        then for some $k>0$,
        \[\nabla_\theta^2 K(\theta)\succeq k\cdot \mathbf{I},\quad \forall \theta\in \Theta.
        \]
        \item \label{itm:envelope} The gradient and Hessian of $\log p(\cdot\mid \theta)$ are both uniformly dominated by integrable functions, i.e.,
        \[
        \sup_{\theta\in \Theta}\|\nabla_\theta \log p(\eta\mid\theta)\|\leq G(\eta),\quad \int_{\eta\in [0,1]} G(\eta)\cdot p(\eta\mid \theta^*) \text{d} \eta<\infty,
        \]
        \[
        \sup_{\theta\in \Theta}\|\nabla_\theta^2 \log p(\eta\mid \theta)\|\leq H(\eta),\quad \int_{\eta\in [0,1]} H(\eta)\cdot p(\eta\mid \theta^*)\text{d} \eta<\infty.
        \]
     \end{enumerate}
\end{assumption}

Assumption \ref{assump:regular_p} is satisfied by minimal regular finite-dimensional exponential families with common support $(0, 1)$, of the form:
\[
p(\eta\mid \theta) = h(\eta) \exp \left(\theta^\top T(\eta) - A(\theta)\right),
\]
where the compact and convex parameter set $\Theta$ lies in the interior of the natural parameter space, the sufficient statistics $T(\eta)$ are integrable, and the Fisher-information is uniformly positive definite over $\Theta$. This class includes the beta family when $(\alpha, \beta)$ is restricted to a compact set like $[\epsilon, M]^2$.

The following lemma controls that the stationary point of $L^*(\theta)$ asymptotically equals $\theta^*$.

\begin{lemma}\label{lem:inspiring_lemma}
    The stationary points of $L^*(\theta)$ converges to $\theta^*$ as $m,n\to\infty$.
\end{lemma}
\begin{proof}
    By Assumption \ref{assump:regular_p} \ref{itm:identify} and \ref{itm:KLconvex}, $\theta^*$ is the unique minimizer of function $K(\theta)$. For any $\theta$ such that $\|\theta - \theta^*\|>r$, there is
    \[
    \|\nabla_\theta K(\theta)\| \geq \dfrac{k r}{2}.
    \]
    On the other hand, from the form of $\nabla_\theta L^*(\theta)$ in \eqref{eq:form_of_grad_Lstar} and
    \[
    \dfrac{1}{2}\log (\mu(1-\mu))\;\leq\; \bar{l}(\eta\mid \eta^*)\;\leq\; \mu\log \mu + (1-\mu) \log (1-\mu),
    \]
    we have
    \begin{equation}\label{eq:asymptotic_exp}
    \lim_{n\to\infty} \sup_{\eta^*, \theta\in \Theta} \Bigg\|\dfrac{\int_{\eta\in [0,1]} \exp\Big(n\cdot  \bar{l}\left(\eta\mid \eta^*\right)\Big) \nabla_\theta p(\eta\mid \theta) \text{d} \eta}{\int_{\eta\in [0,1]} \exp\Big(n\cdot  \bar{l}\left(\eta\mid \eta^*\right)\Big) p(\eta\mid \theta) \text{d} \eta} - \dfrac{\nabla_\theta p(\eta^*\mid\theta)}{p(\eta^*\mid \theta)}\Bigg\| \to 0
    \end{equation}
    and consequently $\lim_{n\to\infty} \sup_{\theta\in \Theta}\| \frac{1}{m}\nabla_\theta L^*(\theta) + \nabla_\theta K(\theta)\| \to 0$. This finishes the proof.

\end{proof}
Inspired by Lemma \ref{lem:inspiring_lemma}, if we can establish that $\frac{1}{m}\nabla L(\theta)$ converges uniformly to $\nabla_\theta K(\theta)$, then, by analogous arguments, it follows that the stationary points of $L(\theta)$ converge in probability to $\theta^*$. The next lemma formalizes this result.

\begin{lemma}\label{lem:most_important}
    If there exists $\gamma>0$ such that $m = o(n^{\gamma})$, then $\frac{1}{m}\nabla_\theta L(\theta)$ converges in probability to $\nabla_\theta K(\theta)$ uniformly over all $\theta\in \Theta$.
\end{lemma}
\begin{proof}
    Let $\eta_1^*,\ldots, \eta_m^*$ denote the true attentiveness levels of each user, which are themselves random i.i.d. random variables. Based on Corollary \ref{cor:corollary_1}, if $m = o(n^{\gamma})$, then with probability $1-o(1)$,
    \[
    \sup_{1\leq j\leq m, 0\leq \eta_1,\eta_2,\ldots,\eta_j\leq 1} \Big| \frac{1}{n}\sum_{i=1}^n l(z_i^{(j)}|\mu ,\eta_j) -  \bar{l}(\eta_j | \eta_j^*)\Big| \leq C\cdot \sqrt{\frac{\log n}{n}}
    \]
    for some constant $C$. This implies the following convergence in probability result: conditional on any realization of $\eta_1^*,\ldots, \eta_m^*$, we have
    \[
    \sup_{1\leq j\leq m, \theta\in \Theta} \Bigg\|\dfrac{\int_{\eta_j\in [0,1]} \exp\Big(\sum_{i=1}^{n} l\left(z_{i}^{(j)}\big\vert\mu, \eta_j\right)\Big) \nabla_\theta p(\eta_j\mid \theta) \text{d} \eta_j}{\int_{\eta_j\in [0,1]} \exp\Big(\sum_{i=1}^{n} l\left(z_{i}^{(j)}\big\vert\mu, \eta_j\right)\Big) p(\eta_j\mid \theta) \text{d} \eta_j} - \dfrac{\int_{\eta_j\in [0,1]} \exp\Big(n\cdot \bar{l}(\eta_j|\eta_j^*)\Big) \nabla_\theta p(\eta_j\mid \theta) \text{d} \eta_j}{\int_{\eta_j\in [0,1]} \exp\Big(n\cdot \bar{l}(\eta_j|\eta_j^*)\Big) p(\eta_j\mid \theta) \text{d} \eta_j}\Bigg\| \stackrel{p}{\to} 0.
    \]
    With $n\to\infty$, this result further implies that
    \[
    \sup_{\theta\in \Theta} \Bigg\|\frac{1}{m}\nabla_\theta L(\theta) - \frac{1}{m}\sum_{j=1}^m \dfrac{\nabla_\theta p(\eta_j^*\mid \theta)}{p(\eta_j^*\mid \theta)}\Bigg\| \stackrel{p}{\to} 0.
    \]
    It remains to show that, with respect to the randomness of $\eta_1^*,\ldots, \eta_m^*$ (which are i.i.d. samples from $\mathcal{P}_\eta$), there is
    \[
    \sup_{\theta\in \Theta} \Bigg\|\frac{1}{m}\sum_{j=1}^m \dfrac{\nabla_\theta p(\eta_j^*\mid \theta)}{p(\eta_j^*\mid \theta)} - \nabla_\theta K(\theta)\Bigg\| \stackrel{p}{\to} 0.
    \]
    By Lemma \ref{lem:ULLN-Lipschitz}, a sufficient condition for this result to hold is that the function family
    \[
    \mathcal{F} = \{\nabla_\theta \log p(\cdot \mid \theta) \,:\, \theta \in \Theta\}
    \]
    is both enveloped and uniformly Lipschitz continuous. These two conditions are ensured by Assumption \ref{assump:regular_p}\ref{itm:envelope}.

\end{proof}

Combining Lemma \ref{lem:inspiring_lemma} and \ref{lem:most_important}, Theorem \ref{thm:simple_two_mass_converge} is proved, which we restate more formally below.

\begin{theorem}[Formal Restatement of Theorem~\ref{thm:simple_two_mass_converge}]
    Let $\hat{\theta}$ denote the converged parameter estimate obtained as the output of Algorithm \ref{alg:EM_UBM} and $\theta^*$ denote the true parameter. If $\mathcal{P}_\eta$ admits a continuous density, then, under Assumptions~\ref{ass:mu} and~\ref{assump:regular_p}, for any $\epsilon > 0$,
    \[
        \lim_{m,\, n_j \to \infty} \mathbb{P}\bigl( \|\hat{\theta} - \theta^*\| \ge \epsilon \bigr) = 0.
    \]
\end{theorem}

\subsubsection{$\mathcal{P}_\eta$ from Example 1}
In Example 1, $\mathcal{P}_\eta$ follows a discrete distribution. As described in Appendix \ref{subapx:extend_two_point_discrete}, we implement a reparameterization trick and use $h_j\in \{0,1\}$ to indicate whether user $j$ is attentive or not. The sample log-likelihood is
\[
L(\theta) = \sum_{j=1}^m \log \Bigg(q_1 \exp\Big(\sum_{i=1}^n l(z_i^{(j)}|\mu, \underline{\eta})\Big) + q_2\exp\Big(\sum_{i=1}^n l(z_i^{(j)}|\mu, \bar{\eta})\Big)\Bigg),
\]
where the parameters are $\theta = (q_1, \underline{\eta}, \bar{\eta})$. The population log-likelihood counterpart is
\[
\begin{aligned}
    L^*(\theta) &= m\cdot \Bigg(q_1^* \log \Big(q_1 \exp\Big(n \cdot\bar{l}(\underline{\eta}| \underline{\eta}^*)\Big) + q_2 \exp\Big(n\cdot\bar{l}(\bar{\eta}|\underline{\eta}^*)\Big)\Big) \\
    &\phantom{= m\cdot \Bigg(} + q_2^* \log \Big(q_1 \exp\Big(n \cdot\bar{l}(\underline{\eta}|\bar{\eta}^*)\Big) + q_2 \exp\Big(n\cdot\bar{l}( \bar{\eta}|\bar{\eta}^*)\Big)\Big)\Bigg),
\end{aligned}
\]
where $\theta^* = (q_1^*, \underline{\eta}^*, \bar{\eta}^*)$ are the true model parameters and $\bar{l}(\eta|\eta^*)$ is defined in \eqref{eq:def_l_bar}. Define function $g(\eta):=1/2 + \eta\cdot (\mu - 1/2)$ and
\[
S_j(\eta) := \sum_{i=1}^n l(z_i^{(j)}|\mu, \eta),\quad S^*(\eta|\eta^*) := n\cdot \bar{l}(\eta|\eta^*).
\]
The gradients of $L(\theta)$ and $L^*(\theta)$ are:
\[
\begin{aligned}
    &\dfrac{\partial}{\partial q_1} L(\theta) = \sum_{j=1}^m \dfrac{\exp(S_j(\underline{\eta})) - \exp(S_j(\bar{\eta}))}{q_1 \exp(S_j(\underline{\eta})) +  q_2\exp(S_j(\bar{\eta}))} \\
    & \dfrac{\partial}{\partial \underline{\eta}} L(\theta) = (\mu-1/2)\cdot\sum_{j=1}^m \dfrac{q_1 \exp(S_j(\underline{\eta}))}{q_1 \exp(S_j(\underline{\eta})) +  q_2\exp(S_j(\bar{\eta}))}\cdot \sum_{i=1}^{n}\dfrac{z_i^{(j)} - g(\underline{\eta})}{g(\underline{\eta})(1-g(\underline{\eta}))} \\
    & \dfrac{\partial}{\partial \bar{\eta}} L(\theta) = (\mu-1/2)\cdot\sum_{j=1}^m \dfrac{q_2 \exp(S_j(\bar{\eta}))}{q_1 \exp(S_j(\underline{\eta})) +  q_2\exp(S_j(\bar{\eta}))}\cdot \sum_{i=1}^{n}\dfrac{z_i^{(j)} - g(\bar{\eta})}{g(\bar{\eta})(1-g(\bar{\eta}))} \\
    \end{aligned}
\]

\begin{flalign*}
& \dfrac{\partial}{\partial q_1} L^*(\theta)
= m \cdot \Bigg(
    q_1^* \dfrac{\exp(S^*(\underline{\eta}| \underline{\eta}^*)) - \exp(S^*(\bar{\eta}| \underline{\eta}^*))}
    {q_1 \exp(S^*(\underline{\eta}| \underline{\eta}^*)) + q_2 \exp(S^*(\bar{\eta}| \underline{\eta}^*))}
    + q_2^* \dfrac{\exp(S^*(\underline{\eta}| \bar{\eta}^*)) - \exp(S^*(\bar{\eta}| \bar{\eta}^*))}
    {q_1 \exp(S^*(\underline{\eta}| \bar{\eta}^*)) + q_2 \exp(S^*(\bar{\eta}| \bar{\eta}^*))}
\Bigg) & \\[0pt]
& \dfrac{\partial}{\partial \underline{\eta}} L^*(\theta)
= (\mu-1/2) mn \cdot \Bigg(
    q_1^* \dfrac{q_1 \exp(S^*(\underline{\eta}| \underline{\eta}^*))}
    {q_1 \exp(S^*(\underline{\eta}| \underline{\eta}^*)) + q_2 \exp(S^*(\bar{\eta}| \underline{\eta}^*))}
    \cdot \frac{g(\underline{\eta}^*) - g(\underline{\eta})}{g(\underline{\eta})(1 - g(\underline{\eta}))}
\\
& \phantom{\dfrac{\partial}{\partial \underline{\eta}} \overline{L}_{M, N}(q_1, \underline{\eta}, \bar{\eta}) = A \cdot M N \cdot \Bigg(}
    + q_2^* \dfrac{q_1 \exp(S^*(\underline{\eta}| \bar{\eta}^*))}
    {q_1 \exp(S^*(\underline{\eta}| \bar{\eta}^*)) + q_2 \exp(S^*(\bar{\eta}| \bar{\eta}^*))}
    \cdot \frac{g(\bar{\eta}^*) - g(\underline{\eta})}{g(\underline{\eta})(1 - g(\underline{\eta}))}
\Bigg) & \\[0pt]
& \dfrac{\partial}{\partial \bar{\eta}} L^*(\theta)
= (\mu-1/2) mn \cdot \Bigg(
    q_1^* \dfrac{q_2 \exp(S^*(\bar{\eta}| \underline{\eta}^*))}
    {q_1 \exp(S^*(\underline{\eta}| \underline{\eta}^*)) + q_2 \exp(S^*(\bar{\eta}| \underline{\eta}^*))}
    \cdot \frac{g(\underline{\eta}^*) - g(\bar{\eta})}{g(\bar{\eta})(1 - g(\bar{\eta}))}
\\
& \phantom{\dfrac{\partial}{\partial \bar{\eta}} \overline{L}_{M, N}(q_1, \underline{\eta}, \bar{\eta}) = A \cdot M N \cdot \Bigg(}
    + q_2^* \dfrac{q_2 \exp(S^*(\bar{\eta}| \bar{\eta}^*))}
    {q_1 \exp(S^*(\underline{\eta}| \bar{\eta}^*)) + q_2 \exp(S^*(\bar{\eta}| \bar{\eta}^*))}
    \cdot \frac{g(\bar{\eta}^*) - g(\bar{\eta})}{g(\bar{\eta})(1 - g(\bar{\eta}))}
\Bigg) &
\end{flalign*}
By Corollary \ref{cor:corollary_1}, we can directly verify that $\nabla_\theta L(\theta)$ converges in probability to $\nabla_\theta L^*(\theta)$ uniformly over all $\theta\in \Theta$. The next theorem shows that the stationary points of $L^*(\theta)$ must be close to $\theta^*$.

\begin{lemma}\label{lemma:cvg_of_BarL_stationary}
For any $r > 0$, there exist constants $m_0$ and $n_0$ such that, for all $m > m_0$ and $n>n_0$, the stationary points of $L^*(\theta)$ are confined to the region
\[
\Theta_0 = \Bigl\{ (q_1, \underline{\eta}, \bar{\eta}) : \underline{\eta} < \bar{\eta}, \;
\max\{|q_1 - q_1^*|, \, |\underline{\eta} - \underline{\eta}^*|, \, |\bar{\eta} - \bar{\eta}^*|\} \leq r \Bigr\}.
\]
\end{lemma}

\begin{proof}
We consider two cases: $|q_1 - q_1^*| > r$ and $|\underline{\eta} - \underline{\eta}^*| > r$.  
The third case $|\bar{\eta} - \bar{\eta}^*| > r$ follows symmetrically and is omitted for brevity.  
Throughout, we assume that $\underline{\eta}$ and $\bar{\eta}$ are sufficiently close to $\underline{\eta}^*$ and $\bar{\eta}^*$, respectively, and that
\[
S^*(\underline{\eta}\,|\,\underline{\eta}^*) > S^*(\bar{\eta}\,|\,\underline{\eta}^*), 
\qquad
S^*(\bar{\eta}\,|\,\bar{\eta}^*) > S^*(\underline{\eta}\,|\,\bar{\eta}^*).
\]
\textbf{Case (1).} Suppose $|q_1 - q_1^*| > r$.  
Define the auxiliary functions
\[
\psi_1(\underline{\eta}, \bar{\eta}) 
:= g(\underline{\eta}^*) \log \!\frac{g(\underline{\eta})}{g(\bar{\eta})}
+ (1 - g(\underline{\eta}^*)) \log \!\frac{1 - g(\underline{\eta})}{1 - g(\bar{\eta})} > 0,
\]
\[
\psi_2(\underline{\eta}, \bar{\eta}) 
:= g(\bar{\eta}^*) \log \!\frac{g(\bar{\eta})}{g(\underline{\eta})}
+ (1 - g(\bar{\eta}^*)) \log \!\frac{1 - g(\bar{\eta})}{1 - g(\underline{\eta})} > 0.
\]
Then
\[
S^*(\underline{\eta}\,|\,\underline{\eta}^*) - S^*(\bar{\eta}\,|\,\underline{\eta}^*) 
= n \cdot \psi_1(\underline{\eta}, \bar{\eta}), 
\qquad
S^*(\bar{\eta}\,|\,\bar{\eta}^*) - S^*(\underline{\eta}\,|\,\bar{\eta}^*) 
= n \cdot \psi_2(\underline{\eta}, \bar{\eta}).
\]

The gradient of $L^*(\theta)$ with respect to $q_1$ can be written as
\[
\begin{aligned}
\Bigg|\dfrac{\partial}{\partial q_1} L^*(\theta)\Bigg|
&= m \cdot \Bigg| 
q_1^* \dfrac{1 - e^{-n \psi_1(\underline{\eta}, \bar{\eta})}}
{q_1 + q_2 e^{-n \psi_1(\underline{\eta}, \bar{\eta})}}
- q_2^* \dfrac{1 - e^{-n \psi_2(\underline{\eta}, \bar{\eta})}}
{q_2 + q_1 e^{-n \psi_2(\underline{\eta}, \bar{\eta})}}
\Bigg| \\[3pt]
&= m \cdot \Bigg|\dfrac{q_1^*}{q_1} - \dfrac{q_2^*}{q_2}\Bigg| + o(1) \\
&\ge m r + o(1).
\end{aligned}
\]
Hence, as $m, n \to \infty$, $\big|\frac{\partial L^*(\theta)}{\partial q_1}\big|$ diverges at rate $O(m)$, and therefore cannot vanish.

\textbf{Case (2).} Suppose $|\underline{\eta} - \underline{\eta}^*| > r$.  
The gradient of $L^*(\theta)$ with respect to $\underline{\eta}$ satisfies
\[
\begin{aligned}
\Bigg|\dfrac{\partial}{\partial \underline{\eta}} L^*(\theta)\Bigg|
&= (\mu - \tfrac{1}{2}) \, m n \cdot \Bigg|
q_1^* \dfrac{q_1}{q_1 + q_2 e^{-n \psi_1(\underline{\eta}, \bar{\eta})}}
\cdot \frac{g(\underline{\eta}^*) - g(\underline{\eta})}{g(\underline{\eta})(1 - g(\underline{\eta}))}
\\[-1pt]
&\phantom{=} \quad
+ q_2^* \dfrac{q_1 e^{-n \psi_2(\underline{\eta}, \bar{\eta})}}{q_1 e^{-n \psi_2(\underline{\eta}, \bar{\eta})} + q_2}
\cdot \frac{g(\bar{\eta}^*) - g(\underline{\eta})}{g(\underline{\eta})(1 - g(\underline{\eta}))}
\Bigg| \\[3pt]
&= (\mu - \tfrac{1}{2}) \, m n \cdot 
\Bigg| q_1^* \frac{g(\underline{\eta}^*) - g(\underline{\eta})}
{g(\underline{\eta})(1 - g(\underline{\eta}))} \Bigg| + o(n) \\
&\ge (\mu - \tfrac{1}{2}) \, m n \cdot 
\dfrac{(1 - \epsilon_q)\, r}{g(1 - \epsilon_\eta)\,[1 - g(1 - \epsilon_\eta)]} + o(n),
\end{aligned}
\]
for some small constants $\epsilon_q, \epsilon_\eta > 0$.  
Thus, as $m,n \to \infty$, $\big|\frac{\partial L^*(\theta)}{\partial \underline{\eta}}\big|$ diverges at rate $O(mn)$ and cannot equal zero.

The same argument applies to $\bar{\eta}$.  
Therefore, any stationary point of $L^*(\theta)$ must lie within the neighborhood $\Theta_0$ defined above.
\end{proof}

\subsection{Convergence under Unknown $\mu$}\label{subsec:cvg_unknown_mu}
Previous derivations in Appendix~\ref{subsec:prove_theorem_two_mass} are all made under Assumption~\ref{ass:mu}, where $\mu_j \equiv \mu$ and $\mu$ is known. In this section, we partially relax this assumption by no longer requiring $\mu$ to be known.

\begin{assumption}\label{assump:unknown_mu}
There exists an unknown $\mu$ such that $\mu_1=\cdots=\mu_m=\mu$.
\end{assumption}

Let $p(\eta \mid \phi)$ denote the density function of $\mathcal{P}_\eta$, and let $\theta \coloneqq (\phi,\mu)$ collect the unknown parameters. The goal of this section is to show that, under certain conditions, the true parameter $\theta^* \coloneqq (\phi^*,\mu^*)$ can be recovered by Algorithm~\ref{alg:EM_UBM}. Similar to Appendix \ref{subsubsec:continuous_eta}, define the log likelihood function:
\[
L(\theta) = \sum_{j=1}^m \log \Bigg( \int_{\eta_j\in[0,1]} \exp\Big(\sum_{i=1}^{n} l\left(z_{i}^{(j)}\big\vert\mu, \eta_j\right)\Big) p(\eta_j\mid \phi) \text{d} \eta_j \Bigg)
\]
and its population version
\[
L^*(\theta) = m \int_{\eta^*\in [0,1]} \log \Bigg( \int_{\eta\in[0,1]} \exp\Big(n\cdot  \bar{l}\left(\eta, \mu\mid \eta^*\right)\Big) p(\eta\mid \phi)  \text{d}\eta \Bigg) p(\eta^*\mid \phi^*)\text{d} \eta^*,
\]
where
\begin{equation}
    \bar{l}(\eta, \mu\mid \eta^*) := (1/2 + \eta^*\cdot (\mu^*-1/2)) \log (1/2 + \eta\cdot (\mu-1/2)) + (1/2 - \eta^*\cdot (\mu^*-1/2))\log (1/2 - \eta\cdot (\mu-1/2)).
\end{equation}
The gradients of $L(\theta)$ and $L^*(\theta)$ are:
\begin{equation}
\begin{aligned}
    \nabla_\phi L(\theta) &= \sum_{j=1}^m \dfrac{\int_{\eta_j\in [0,1]} \exp\Big(\sum_{i=1}^{n} l\left(z_{i}^{(j)}\big\vert\mu, \eta_j\right)\Big) \nabla_\phi p(\eta_j\mid \phi) \text{d} \eta_j}{\int_{\eta_j\in [0,1]} \exp\Big(\sum_{i=1}^{n} l\left(z_{i}^{(j)}\big\vert\mu, \eta_j\right)\Big) p(\eta_j\mid \phi) \text{d} \eta_j}\\
    \frac{\mathrm{d}}{\mathrm{d} \mu} L(\theta) &= \sum_{j=1}^m \dfrac{\int_{\eta_j\in [0,1]} \exp\Big(\sum_{i=1}^{n} l\left(z_{i}^{(j)}\big\vert\mu, \eta_j\right)\Big) p(\eta_j\mid \phi)\cdot \Big(\sum_{i=1}^n \frac{\mathrm{d}}{\mathrm{d}\mu} l(z_i^{(j)}|\mu, \eta_j)\Big) \text{d} \eta_j}{\int_{\eta_j\in [0,1]} \exp\Big(\sum_{i=1}^{n} l\left(z_{i}^{(j)}\big\vert\mu, \eta_j\right)\Big) p(\eta_j\mid \phi) \text{d} \eta_j}\\
    \nabla_\phi L^*(\theta) &=m\cdot \int_{\eta^*\in [0,1]} \Bigg(\dfrac{\int_{\eta\in [0,1]} \exp\Big(n\cdot  \bar{l}\left(\eta, \mu\mid \eta^*\right)\Big) \nabla_\phi p(\eta\mid \phi) \text{d} \eta}{\int_{\eta\in [0,1]} \exp\Big(n\cdot  \bar{l}\left(\eta, \mu\mid \eta^*\right)\Big) p(\eta\mid \phi) \text{d} \eta} \Bigg)\, p(\eta^*\mid \phi^*)\text{d} \eta^*\\
    \frac{\mathrm{d}}{\mathrm{d} \mu} L^*(\theta) &=m\cdot \int_{\eta^*\in [0,1]} \Bigg(\dfrac{\int_{\eta\in [0,1]} \exp\Big(n\cdot  \bar{l}\left(\eta, \mu\mid \eta^*\right)\Big) p(\eta\mid \phi) \cdot \Big(n\cdot \frac{\mathrm{d}}{\mathrm{d}\mu} \bar{l}(\eta, \mu \mid \eta^*)\Big)\text{d} \eta}{\int_{\eta\in [0,1]} \exp\Big(n\cdot  \bar{l}\left(\eta, \mu\mid \eta^*\right)\Big) p(\eta\mid \phi) \text{d} \eta} \Bigg)\, p(\eta^*\mid \phi^*)\text{d} \eta^*\\
\end{aligned}
\end{equation}
where
\[
\begin{aligned}
    \frac{\mathrm{d}}{\mathrm{d}\mu} l(z_i^{(j)}|\mu, \eta_j) &= z_{i}^{(j)}\cdot \frac{\eta_j}{1/2 + \eta_j\cdot (\mu - 1/2)} - (1-z_{i}^{(j)})\cdot \frac{\eta_j}{1/2 - \eta_j\cdot (\mu - 1/2)}\\
    \frac{\mathrm{d}}{\mathrm{d}\mu}\bar{l}(\eta, \mu\mid \eta^*) &= \eta\cdot \Big(\frac{1/2 + \eta^*\cdot (\mu^*-1/2)}{1/2 + \eta\cdot (\mu-1/2)} - \frac{1/2 - \eta^*\cdot (\mu^*-1/2)}{1/2 - \eta\cdot (\mu-1/2)}\Big)
\end{aligned}
\]
Similar to the proof of Lemma~\ref{lem:inspiring_lemma}, suppose that $p(\eta\mid \phi)$ satisfies Assumption~\ref{assump:regular_p}. Then, by the same reasoning that leads to the asymptotic expression in~\eqref{eq:asymptotic_exp}, as $n\to\infty$, we have
\[
\frac{1}{m} \nabla_\phi L^*(\theta) \to \int_{\eta^*\in [0,1]} \frac{\nabla_\phi p(\eta_\mu^*\mid \phi)}{p(\eta_\mu^*\mid \phi)}\cdot  p(\eta^*\mid \phi^*) \mathrm{d} \eta^*
\]
where
\[
\eta_\mu^* \coloneqq \min \left\{\eta^*\cdot \frac{\mu^*-1/2}{\mu-1/2}, 1\right\}.
\]
When $\mu=\mu^*$, the expression matches exactly the analysis in Lemma~\ref{lem:inspiring_lemma}. It remains to guarantee that, for every $\epsilon>0$, there exists $\delta>0$ such that whenever $\|\mu-\mu^*\|>\epsilon$, we have $\|\frac{1}{m}\nabla_\phi L^*(\theta)\| >\delta$. The following assumption formalizes the requirement.

\begin{assumption}\label{assump:identifiable}
Let $\Phi$ and $\mathcal U\subset(1/2,1)$ be compact parameter spaces for
$\phi$ and $\mu$, with $\phi^*\in\Phi$ and $\mu^*\in\mathcal U$. For
$\mu\in\mathcal U$, define
\[
h(\eta,\mu)
\coloneqq
\min\left\{
\eta\cdot \frac{\mu^*-1/2}{\mu-1/2},\,1
\right\}.
\]
Then for every $\epsilon>0$, the following inequality holds,
\[
\inf_{\phi\in\Phi,\ \mu\in\mathcal U,\ |\mu-\mu^*|\ge \epsilon}
\left\|
\int_0^1
\nabla_\phi \log p(h(\eta, \mu)\mid\phi)
\,p(\eta\mid\phi^*)\,d\eta
\right\|>0.
\]
\end{assumption}
Intuitively, this assumption requires that after rescaling/clipping the true
attentiveness distribution using an incorrect value of $\mu$, the resulting
distribution cannot be represented by any member of the assumed family $\mathcal P_\eta$. For a natural exponential
family, this condition is equivalent to requiring the transformed sufficient-statistic
moment vector $\mathbb E_{\phi^*}[T(h(\eta, \mu))]$ to be uniformly separated from the mean-parameter set $\{\nabla A(\phi):\phi\in\Phi\}$ whenever $|\mu-\mu^*|\ge\epsilon$. This condition can be checked numerically for a specified parametric family and
compact parameter space. The beta family with $(\alpha, \beta)$ restricted to $[\epsilon, M]^2$ satisfies this assumption.

The following theorem is a generalization of Theorem \ref{thm:simple_two_mass_converge} to the unknown $\mu$ setting.
\begin{theorem}\label{thm:cvg_unknown_mu}
    Let $\hat{\theta}\coloneqq (\hat{\phi}, \hat{\mu})$ denote the converged parameter estimate obtained as the output of Algorithm \ref{alg:EM_UBM} and $\theta^*$ denote the true parameter. If $\mathcal{P}_\eta$ admits a continuous density, then, under Assumptions~\ref{assump:regular_p}, \ref{assump:unknown_mu} and~\ref{assump:identifiable}, for any $\epsilon > 0$,
    \[
        \lim_{m,\, n_j \to \infty} \mathbb{P}\bigl( \|\hat{\theta} - \theta^*\| \ge \epsilon \bigr) = 0.
    \]
\end{theorem}
\begin{proof}
   The proof proceeds almost identically to the proofs of Lemmas~\ref{lem:inspiring_lemma} and~\ref{lem:most_important}, so we omit most repetitive steps. The first claim is that the stationary points of $L^*(\theta)$ converge to $\theta^*$ as $m,n\to\infty$. To establish this claim, it suffices to show that, for any $\hat{\theta}$ satisfying $\|\hat{\theta}-\theta^*\|>\epsilon$, there exists some $\delta>0$ such that
\[
\liminf_{m,n\to\infty}
\left\|
\frac{1}{m}\nabla_\theta L^*(\hat{\theta})
\right\|
> \delta .
\]
Lemma~\ref{lem:inspiring_lemma} already guarantees this condition in the case where $\hat{\theta}=(\hat{\phi},\mu^*)$. On the other hand, Assumption~\ref{assump:identifiable} guarantees the condition whenever $|\hat{\mu}-\mu^*|>\epsilon$. Combining these two cases yields the desired claim.

The second step is to bridge the gap between $\nabla_\theta L^*(\hat{\theta})$ and $\nabla_\theta L(\hat{\theta})$. By an argument similar to that used in Corollary~\ref{cor:corollary_1}, when $m,n\to\infty$ with $m=o(n^\gamma)$, for any fixed and sufficiently small $\delta_\mu>0$, with probability $1-o(1)$,
\[
\sup_{\substack{
1\leq j\leq m,\; 0\leq \eta_1,\eta_2,\ldots,\eta_j\leq 1,\\
1/2+\delta_\mu \leq \mu\leq 1-\delta_\mu
}}
\left|
\frac{1}{n}\sum_{i=1}^n l(z_i^{(j)}\mid \mu,\eta_j)
-
\bar{l}(\eta_j\mid \eta_j^*)
\right|
\leq
C\sqrt{\frac{\log n}{n}},
\]
for some constant $C>0$. Therefore, for any compact parameter space $\Theta=\Phi\times [1/2+\delta_\mu,1-\delta_\mu]$, we have
\[
\sup_{\theta\in\Theta}
\left\|
\frac{1}{m}\nabla_\theta L(\theta)
-
\frac{1}{m}\nabla_\theta L^*(\theta)
\right\|
\stackrel{p}{\to} 0
\]
Combining these two steps finishes the whole proof.
\end{proof}

\end{document}